\documentclass[twoside]{article}

\usepackage[accepted]{aistats2020}
\usepackage[round]{natbib}

\usepackage{algorithm}
\usepackage[noend]{algpseudocode}
\usepackage{amsmath}
\usepackage{amsthm}
\usepackage{amssymb}
\usepackage{amsfonts}
\usepackage{comment}
\usepackage{dsfont}
\usepackage{hyperref}
\newtheorem{theorem}{Theorem}
\newtheorem{lemma}{Lemma}
\newtheorem*{assumption*}{Assumption}

\newtheorem{remark}{Remark}
\usepackage{graphicx}
\usepackage{epstopdf}
\usepackage{tcolorbox}
\usepackage{caption}
\usepackage{subcaption}
\usepackage{tabularx}
\usepackage{threeparttable}
\usepackage{arydshln}
\usepackage{multirow}
\usepackage{enumitem}
\usepackage{booktabs}
\usepackage{mathtools, cuted}

\usepackage{amsmath,amsfonts,bm}

\def\eqref#1{equation~\ref{#1}}

\def\1{\bm{1}}

\def\va{{\bm{a}}}

\def\vw{{\bm{w}}}
\def\vx{{\bm{x}}}
\def\vy{{\bm{y}}}

\def\mA{{\bm{A}}}

\def\mI{{\bm{I}}}

\def\mW{{\bm{W}}}

\def\mZ{{\bm{Z}}}

\DeclareMathAlphabet{\mathsfit}{\encodingdefault}{\sfdefault}{m}{sl}
\SetMathAlphabet{\mathsfit}{bold}{\encodingdefault}{\sfdefault}{bx}{n}

\usepackage{enumitem, hyperref}
\makeatletter
\def\namedlabel#1#2{\begingroup
    #2%
    \def\@currentlabel{#2}%
    \phantomsection\label{#1}\endgroup
}
\makeatother

\begin{document}
%

%
\twocolumn[
\aistatstitle{Linear Convergence of Adaptive Stochastic Gradient Descent}
\aistatsauthor{Yuege Xie$^*$  \And Xiaoxia Wu$^{\dagger}$ \And Rachel Ward$^{*\dagger}$}
\aistatsaddress{ $^*$Oden Institute, UT Austin \And $^{\dagger}$Department of Mathematics, UT Austin}
]
\begin{abstract}
   We prove that the norm version of the adaptive stochastic gradient method (AdaGrad-Norm) achieves a linear convergence rate for a subset of either strongly convex functions or non-convex functions that satisfy the Polyak-\L{}ojasiewicz (PL) inequality. The paper introduces the notion of Restricted Uniform Inequality of Gradients (RUIG)---which is a measure of the balanced-ness of the stochastic gradient norms---to depict the landscape of a function.
   RUIG plays a key role in proving the robustness of AdaGrad-Norm to its hyper-parameter tuning in the stochastic setting. 
   On top of RUIG, we develop a two-stage framework to prove the linear convergence of AdaGrad-Norm without knowing the parameters of the objective functions. This framework can likely be extended to other adaptive stepsize algorithms. The numerical experiments validate the theory and suggest future directions for improvement.
\end{abstract}


\section{Introduction}
Consider the optimization problem of 
minimizing the empirical risk: 
$$ \min_{\vx \in \mathbb{R}^d}F(\vx):= \frac{1}{n} \sum_{i=1}^n f_i(\vx)$$
where $f_{i}(\vx) = f(\vx, \mZ_i): \mathbb{R}^{d} \rightarrow \mathbb{R}, i = 1, 2, \ldots$ and $\{\mZ_1,\dots, \mZ_n\}$ are empirical samples drawn uniformly from an unknown underlying distribution $\mathcal{S}$. In this paper, we focus on smooth functions $F(\vx)$ that are either strongly convex, or non-convex with Polyak-\L{}ojasiewicz inequality \citep{lojasiewicz1963propriete, polyak1963gradient}, which are fundamental to a variety of machine learning problems \citep{bottou2004large, bottou2018optimization}. 

Linear convergence results using stochastic gradient descent (SGD) or accelerated SGD \citep{bottou-91a, nash1991numerical, bertsekas1999nonlinear, nesterov2005smooth, haykin2005cognitive, bubeck2015convex} to solve the above problem have been established for this class of functions: SGD with fixed stepsize guarantees linear convergence to global minima \citep{allen2018convergence,zou2018stochastic} or up to a radius around the optimal solution \citep{moulines2011non, Needell2016}; Improved algorithms---like SAG \citep{schmidt2017minimizing}, SVRG \citep{johnson2013accelerating} and SAGA \citep{defazio2014saga}---allow faster linear convergence to the global minimizer. However, since the above convergence requires that fixed stepsizes must meet a certain threshold determined by unknown parameters such as the level of stochastic noise, Lipschitz smoothness constants, and strong convexity parameters, SGD and variance reduced SGD are highly sensitive to stepsize tuning in practice. Thus, seeking an algorithm that is robust to the choice of hyper-parameters is as crucial as designing an algorithm that gives faster convergence. The paper focuses on the robustness of the linear convergence of adaptive stochastic gradient descent to unknown hyperparameters.
  
Adaptive gradient descent methods introduced in \cite{duchi2011adaptive} and \cite{mcmahan2010adaptive} update the stepsize on the fly: They either adapt a vector of per-coefficient stepsizes \citep{kingma2014adam,lafond-vasilache-bottou-2017,Reddi2018OnTC, shah2018minimum,zou2018sufficient,staib2019escaping} or a single stepsize depending on the norm of the gradient \citep{levy2017online,ward2018adagrad, wu2018wngrad}.  
The latter one, AdaGrad-Norm \citep{ward2018adagrad} 
 has the following updates: 
\begin{align*}
    b_{j+1}^2 &= b_j^2 + \|\nabla f_{\xi_j}(\vx_j)\|^2; \\
     \vx_{j+1} & = \vx_{j} -\frac{\eta}{b_{j+1}} \nabla f_{\xi_j}(\vx_j) 
 \end{align*}
where $\xi_j \sim \text{Unif}\{1,2,\dots \}$ such that $\mathds{E}_{
\xi_j}\left[\nabla f_{\xi_j}(\vx)\right]= \nabla F(\vx), \forall \vx.$
AdaGrad-Norm has
been shown to be extremely resilient to the functions' parameters being unknown \citep{levy2017online, levy2018online, ward2018adagrad}.  
In addition to this robustness, AdaGrad-Norm enjoys $\mathcal{O}(1/\sqrt{T})$ convergence rate for smooth non-convex functions under the metric $\min_{j\in[T]} \|\nabla F(\vx_j)\|^2$ \citep{ward2018adagrad,li2018convergence}. 
This asymptotic convergence rate has also been proved for general convex functions \citep{levy2018online}.
A linear convergence rate $\mathcal {O}\left( \exp(-\kappa T) \right)$ \footnote{$\kappa$ is the condition number} is possible for strongly convex smooth functions using variants of AdaGrad-Norm in which the final update uses a harmonic sum of the queried gradients  \citep{levy2017online}. 
Yet, the analysis in \cite{levy2017online} and \cite{levy2018online} requires a priori information: a convex set with a known diameter in which the global minimizer resides. 
The analysis in \cite{ward2018adagrad}  considers the smooth function under an assumption of a bounded stochastic gradient norm that rules out the strongly convex cases, while \cite{li2018convergence} only assumes bounded variance but requires prior knowledge of smoothness. Therefore, obtaining a robust linear convergence guarantee without prior knowledge of a convex set or the smoothness parameters, remains an open question for AdaGrad-Norm with strongly convex objectives. 

In this paper, we establish robust linear convergence guarantees for AdaGrad-Norm for strongly convex functions without requiring knowledge of smoothness or strong convexity parameters, nor the knowledge of a convex set containing the minimizer, and we also extend our analysis to non-convex functions that satisfy the Polyak-\L{}ojasiewicz (PL) inequality.\footnote{Note that our results are for the norm version of AdaGrad (AdaGrad-Norm), which differs from the convergence of the diagonal version of AdaGrad and its variants (with momentum) \citep{balles2018dissecting,bernstein_signum,mukkamala2017variants,pmlr-v80-chen18m}.} Our analysis does not follow the standard analysis---which assumes the bounded variance $\hat{\sigma}$ for $\mathds{E}_{\xi_j}[ \|\nabla f_{\xi_{j}}(\vx) -\nabla F(\vx) \|^2] \leq \hat{\sigma}^2, \forall j, \forall \vx$ in \cite{levy2018online, levy2017online,ward2018adagrad,li2018convergence}---and avoids likely sub-linear convergence results. The set of functions for which we guarantee a robust linear convergence rate using AdaGrad-Norm includes certain classes of neural networks. Among these many applications, one function class of particular interest is the over-parameterized neural network \citep{vaswani2018fast, zhang2016understanding,du2018gradient,zhou2018sgd, bassily2018exponential}. Our contributions are not only significant for the algorithm in its own right, but because of the generality of our two-stage framework for the linear convergence proof, we believe it is easily applicable to other adaptive algorithms such as Adam \citep{kingma2014adam} and AMSGrad \citep{Reddi2018OnTC}.

\textbf{Notations} $\|.\|$ denotes the $\ell_2$-norm. $\mu$ is either the $\mu-$strongly convex parameter in Assumption \ref{A1a} or the $\mu-$PL Inequality parameter in Assumption \ref{A1b}. In the batch setting, $L$ is the smallest Lipschitz constant of $\nabla F(\vx)$; in the stochastic setting, $L \triangleq \sup_iL_i$, where $L_i$ is the Lipschitz constant of $\nabla f_i(\vx)$.  $\mathds{P}_i(\cdot)$ is the probability w.r.t. the $i$-th sample point.

\subsection{Main Contributions}
We propose Restricted Uniform Inequality of Gradients (RUIG) to measure the uniform lower bound of stochastic gradients according to $\|\vx-\vx^*\|$ in a restricted region. On top of RUIG, we show that the evolution of the error can be divided into the following two stages:
\begin{itemize}[leftmargin=*]
\item \textbf{Stage I} If $b_t< \eta\mu\leq\eta L$, $\|\vx_t-\vx^*\|^2$ increases first (but remains smaller than a certain upper bound), and contracts after $b_t\geq \eta \mu$, while $b_t$ continues growing until it exceeds $\eta L$;
\item \textbf{Stage II} $b_t >\eta L$, AdaGrad-Norm converges linearly. $b_t$ increases during the optimization process but it is always bounded by $ b_{\max}$.
\end{itemize}

We illustrate these stages in Figure \ref{fig:demo} with $\eta =1$.
\begin{figure}[htbp]
    \centering
     \includegraphics[width=.47\textwidth]{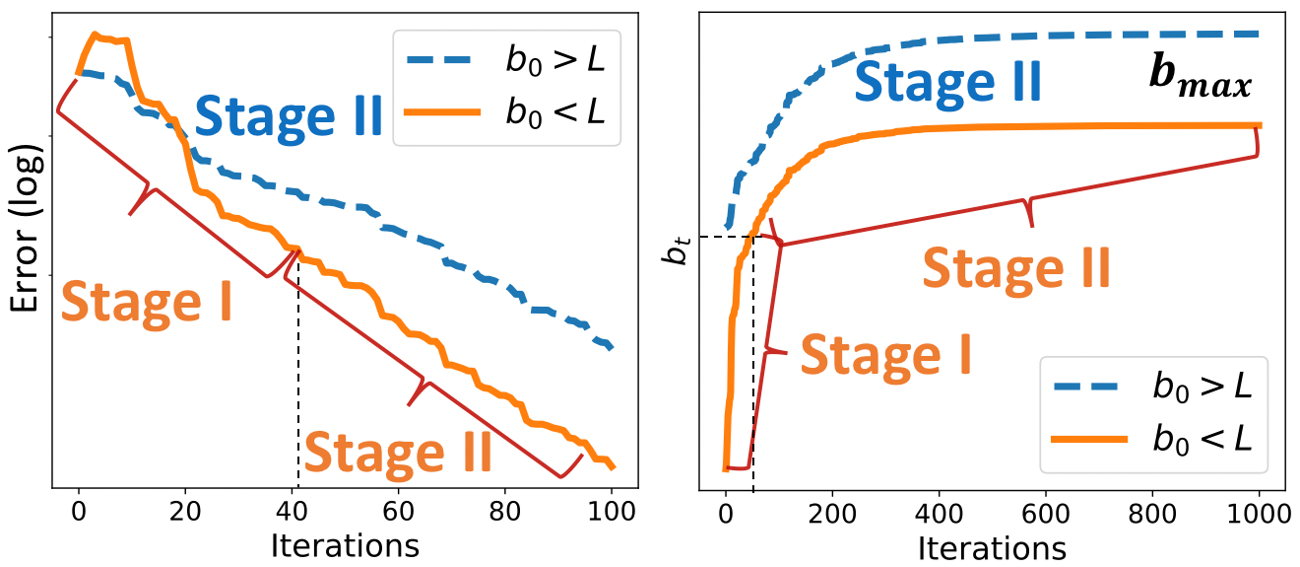}
    \captionof{figure}{Two-Stage Convergence of AdaGrad-Norm with different initial stepsizes: $b_0<L$ versus $b_0>L$. Left: Error $\|\vx_t-\vx^*\|^2$ in logarithmic scale. Right: Growth of $b_t$ to corresponding upper bounds ($\leq b_{\max}$).}
    \label{fig:demo}
\end{figure}


We prove the non-asymptotic linear convergence of AdaGrad-Norm in the strongly convex setting for stochastic and batch updates; furthermore, we also extend our results for non-convex functions satisfying PL inequality. Our main results are as follows (informal):
\begin{enumerate}
    \item In the stochastic setting, Theorem \ref{sgdscalg} shows that AdaGrad-Norm attains $\min_i\|\vx_i-\vx^*\|^2 \leq \epsilon$ with high probability after $ T = \mathcal{O}(\log \frac{1}{\epsilon})$  iterations for $b_0 > \eta L$; and after $T = \mathcal{O}(\frac{1}{\epsilon} + \log \frac{1}{\epsilon})$ iterations for $b_0 \leq \eta L$, assuming that $F(\vx)$ is $\mu-$strong convex, $L-$smooth, almost stationary and with $(\epsilon, \alpha, \gamma)$-RUIG ($\forall \epsilon>0$, for any fixed $\vx \in \mathbb{R}^d$, if $\|\vx-\vx^*\|^2 > \epsilon$, then $\exists (\alpha, \gamma)$ s.t. $\mathds{P}_{\xi_j} (\|\nabla f_{\xi_j}(\vx)\|^2 \geq \alpha \|\vx-\vx^*\|^2) \geq \gamma, \xi_j=1,2,\dots$). 
    
    \item In the batch setting, by using the full gradient, the above probability $\gamma$ degrades to 1 and $\alpha = \mu^2$. Theorem \ref{op2} shows that $\min_i \|\vx_i-\vx^*\|^2 \leq \epsilon$ after $T = \mathcal{O}(\log \frac{1}{\epsilon})$ iterations for $b_0 > \eta \frac{\mu+L}{2}$ and after $T = \mathcal{O}(\frac{1}{\log (1+ C\epsilon)}+ \log \frac{1}{\epsilon})$ iterations for $b_0\leq \eta \frac{\mu+L}{2}$, if $F(x)$ is strongly convex and smooth.
    
    \item For non-convex functions with the PL inequality, we alternatively consider the convergence rate of $\min_i F(\vx_i)-F^*$. Theorem \ref{non-convex} illustrates that $\min_i F(\vx_i)-F^*\leq \epsilon$ after $T=\mathcal{O}(\log\frac{1}{\epsilon})$ iterations for  $b_0 > \eta L$; and after $T = \mathcal{O}(\frac{1}{\log (1+ C\epsilon)}+ \log \frac{1}{\epsilon})$ iterations for $b_0\leq \eta L$.
\end{enumerate}

\begin{table*}[ht]
\centering
\caption{Summary of convergence rates of (stochastic) GD in strongly convex setting}
\label{rates}
\begin{threeparttable}
  \setlength\extrarowheight{4pt}
  \begin{tabular}{clll}
    \toprule
 Setting  & {Algorithm}  &  Initial stepsize  & Iterations to achieve $\|\vx_{best}-\vx^*\|^2\leq \epsilon$ \tnote{1}\\
\midrule
\multirow{3}{*}{Stochastic GD} &  fixed stepsize  & $\eta_0 = \frac{1}{2\sup_i L_i}$  & $\mathcal{O}(\frac{\sup_i L_i}{\mu}\log\frac{\Delta_0}{\epsilon})$ \citep{Needell2016}\\
    &\textbf{AdaGrad-Norm\tnote{2}} & $\eta_0 = \frac{1}{b_0} < \frac{1}{\sup_i L_i}$ & $\mathcal{O}(\frac{\sup_i L\Delta_0}{\mu}\log \frac{\Delta_0}{\epsilon}) $\\
    &\textbf{AdaGrad-Norm} &  arbitrary  & $\mathcal{O}(\frac{1}{\epsilon} + \frac{\sup_i L_i\Delta_0}{\mu} \log \frac{\Delta_0}{\epsilon})$\\[0.2cm]
    \hdashline
    \rule{0pt}{1.3\normalbaselineskip}
   \multirow{3}{*}{Deterministic GD} &  fixed stepsize  & $\eta_0 = \frac{2}{\mu + L}$ & $\mathcal{O}(\frac{(\mu+L)^2}{4\mu L}\log\frac{\Delta_0}{\epsilon})$ \citep{bubeck2015convex}\\
    &\textbf{AdaGrad-Norm} & $\eta_0 = \frac{1}{b_0} < \frac{2}{\mu + L}$  &  $\mathcal{O}(\frac{L\Delta_0}{\mu} \log \frac{\Delta_0}{\epsilon})$ \\
    &\textbf{AdaGrad-Norm} &  arbitrary & $\mathcal{O}(\frac{1}{\log (1+ C\epsilon)} +  \frac{L\Delta_0}{\mu} \log \frac{\Delta_0}{\epsilon})$ \\
    \bottomrule
  \end{tabular}
  \begin{tablenotes}
   \item $^1$ $\Delta_0 = \|\vx_0-\vx^*\|^2$ is the initial distance to the minimizer $\vx^*$. 
   \item $^2$ AdaGrad-Norm with $\eta = 1$. 
  \end{tablenotes}
 \end{threeparttable}
\end{table*}

We show that the convergence is robust starting from any initialization of $b_0$, without knowing the Lipschitz constant or strong convexity parameter a priori. The robustness is shown in Table \ref{rates}: when starting from different initial stepsizes, the convergence rates of AdaGrad-Norm are only changed according to the slope in Stage II and negligible gain from the added-on sublinear part in Stage I. However, changing the initial stepsize for SGD causes divergence.

\section{Problem Setup}\label{sec:prob}
Consider the empirical risk $F(\vx) = \frac{1}{n} \sum_{i=1}^n f_i(\vx)$, where $f_{i}(\vx) = f(\vx, \mZ_i) : \mathbb{R}^{d} \rightarrow  \mathbb{R}, i = 1, 2, \dots, n$ with possibly infinite $n$. In contrast to (stochastic) Gradient Descent implemented with \emph{fixed} stepsize, the update rules of AdaGrad-Norm (see Algorithm \ref{SGDSC}) \emph{dynamically} incorporates the information from previous gradients into the reciprocal of the learning rates.

\begin{algorithm}
\caption{AdaGrad-Norm }\label{SGDSC}
    \begin{algorithmic}
    \State \textbf{Input:} Initialize $\epsilon >0, \eta>0, T>0, \vx_0 \in \mathbb{R}^d, b_0 \in \mathbb{R}, j \leftarrow 0 $
    \While {$ j<T$}
        \State Generate random variable $\xi_{j}$ and compute $G_{j}$ (stochastic: $ G_j = \nabla f_{\xi_{j}}(\vx_{j})$; batch: $G_{j} = \nabla F(\vx_j)$)
          \State $b_{j+1}^2 \leftarrow b_{j}^2 + \|G_{j}\|^2$
          \State $\vx_{j+1} \leftarrow \vx_{j} - \frac{\eta}{b_{j+1}} G_{j}$
        \State $j \leftarrow j+1$
    \EndWhile
    \end{algorithmic}
\end{algorithm}

 The algorithm follows the standard assumptions from \cite{bottou2018optimization}: for each $j \geq 0$, the random vectors $\xi_j$, $j = 0,1,2, \dots, $ are mutually independent, independent of $\vx_j$, and satisfy $\mathds{E}_{\xi_j}[\nabla f_{\xi_{j}}(\vx_{j})]=\nabla F(\vx_j)$. In the stochastic setting, it draws one sample at a time and uses unbiased estimators ($G_j = \nabla f_{\xi_{j}}(\vx_{j})$) of the full gradients of $ F(\vx_j)$ to update. In the batch setting, it uses full gradients ($G_j = \nabla F(\vx_j)$) instead.
 
 In the convergence analysis, we consider the following two equivalent updates of AdaGrad-Norm:
\begin{align*}
\textbf{Square Form: } b_{j+1} &= \sqrt{b_j^2 + \|\nabla f_{\xi_j}(\vx_j)\|^2 } \\
\textbf{Solution Form: } b_{j+1} &= b_j + \frac{ \|\nabla f_{\xi_j}(\vx_j)\|^2}{b_j+b_{j+1}}
\end{align*}

\paragraph{Assumptions}
Throughout the paper, we use different combinations of the following assumptions to analyze the convergence rates in both the stochastic (with Assumptions \ref{A1a}, \ref{A2}, \ref{A3} and \ref{A4}) and batch (with Assumptions \ref{A1a}/\ref{A1b} and \ref{A2}) settings. 
\begin{enumerate}
    \item[\namedlabel{A1a}{(A1a)}] \textbf{$\mu-$strongly convex:} $F(\vx)$ is differentiable and $\langle \nabla F(\vx) - \nabla F(\vy), \vx-\vy \rangle \geq \mu \|\vx-\vy\|^2, \forall \vx,\vy$.
    
    \item[\namedlabel{A1b}{(A1b)}] \textbf{$\mu-$Polyak-\L{}ojasiewicz (PL) Inequality:} $\|\nabla F(\vx)\|^2 \geq 2 \mu (F(\vx) - F(\vx^*)), \forall \vx$.
    
    \item[\namedlabel{A2}{(A2)}] \textbf{$L-$smooth:} $f_i(\vx)$ is $L_i-$smooth, $\forall i$: $\|\nabla f_i(\vx) - \nabla f_i(\vy)\| \leq L_i\|\vx-\vy\|, \forall \vx,\vy$. Let $L\triangleq \sup_i L_i$, $F(\vx)$ and $\{f_i(\vx)\}$ are all $L-$smooth.
    
    \item[\namedlabel{A3}{(A3)}] \textbf{$(\epsilon, \alpha, \gamma)-$Restricted Uniform Inequality of Gradients (RUIG):} $\forall \epsilon>0$, for any fixed $\vx \in \mathcal{D}_{\epsilon} \triangleq \{ \vx\in \mathbb{R}^d: \|\vx-\vx^*\|^2 > \epsilon\}$, $\exists (\alpha, \gamma)$ s.t. $\alpha>0$, $\gamma>0$, and $\mathds{P}_i (\|\nabla f_i(\vx)\|^2 \geq \alpha \|\vx-\vx^*\|^2) \geq \gamma$, $\forall i = 1,2,\dots$.
    
    \item[\namedlabel{A4}{(A4)}] \textbf{convex and almost stationary:} \citep{moulines2011non,Needell2016,vaswani2018fast} $f_i(\vx)$ is convex, $\forall i$. Let $\vx^*= \arg\min_\vx F(\vx)$, then $\mathds{P}_i(\nabla f_i(\vx^*) = \mathbf{0})=1, \forall i$, i.e. the probability  of $\vx^*$ being a stationary point is almost surely over all sample data points.
  
\end{enumerate}

 Assumption \ref{A3} is a sufficient condition to guarantee the linear convergence for AdaGrad-Norm with any initialization of stepsize, but it is not necessary when the initial stepsize is smaller than the unknown critical values, i.e. $\frac{1}{\eta L}$ or $\frac{2}{\eta (\mu+L)}$. Examples of systems with this property are in Section \ref{sec:ruig}.
 
Assumption \ref{A4} is the key condition for linear convergence of $\|\vx-\vx^*\|^2$  in the stochastic approximation algorithms \citep{bach2012a, wu2018wngrad} as it imposes a strong condition on each component function at the point $x^*$. However, this assumption is much weaker than (Strong or Weak\footnote{The weak growth condition in \cite{Cevher2019} is weaker than that in \cite{vaswani2018fast}.}) Growth Condition in \cite{schmidt2013fast,vaswani2018fast} where it is assumed that $\forall \vx \in \mathbb{R}^d$, 
 $\max_i \|\nabla f_i(\vx)\|^2\leq B\| \nabla F(\vx)\|^2 $ or 
 $\mathds{E}_i \|\nabla f_i(\vx)\|^2< B (F(\vx)-F^*) $, for some constant $B$. We use the weaker assumption and characterize a better convergence rate for AdaGrad-Norm over many optimization problems that satisfy Assumption \ref{A4}. One particularly relevant application is the over-parameterized neural network.
 Note that Assumption \ref{A4} implies that there exists almost no noise at the solution, which may not be appropriate for certain applications.

\section{Restricted Uniform Inequality of Gradients}\label{sec:ruig}

In this section, we concretely explain our Assumption \ref{A3} in Section \ref{sec:prob} and restate it as follows:
\begin{assumption*}
\textbf{$(\epsilon, \alpha, \gamma)-$Restricted Uniform Inequality of Gradients (RUIG):} $\forall \epsilon>0$, for any fixed $\vx \in \mathcal{D}_{\epsilon} \triangleq \{\vx\in \mathbb{R}^d: \|\vx-\vx^*\|^2 > \epsilon\}$, $\exists (\alpha, \gamma)$ s.t. $\alpha>0$, $\gamma>0$, and
\begin{align*}
    \mathds{P}_i (\|\nabla f_i(\vx)\|^2 \geq \alpha \|\vx-\vx^*\|^2) \geq \gamma, \forall i = 1,2,\dots
\end{align*}
\end{assumption*}
The RUIG gives a lower bound on the probability $\gamma$, with which the norm of any unbiased gradient estimator $\|\nabla f_i(\vx)\|$ is larger than the distance between $\vx$ and $\vx^*$ by a constant factor $\alpha$, if $\vx$ is in a restricted region $\mathcal{D}_{\epsilon}$. This inequality depicts a set of functions \{$F(\vx)$\} that preserve a flat landscape around $\vx^*$ for each component loss function $f_i(\vx)$ and characterize the relatively sharper curvature beyond the region. 

The constant tuple $(\epsilon, \alpha, \gamma)$ is determined by the distribution of the dataset. In general, $\alpha$ and $\gamma$ are negatively correlated, i.e., $\alpha \rightarrow 0$, $\gamma \rightarrow 1$. The error $\epsilon$ could be independent of $\alpha$ and $\gamma$. However, with large $\epsilon$, the product $\alpha\gamma$ is more likely far away from zero. In addition, if $\epsilon_2 \geq \epsilon_1 \geq 0$, then $\mathcal{D}_{\epsilon_2} \subseteq \mathcal{D}_{\epsilon_1}  \subseteq \mathcal{D}_0 = \mathbb{R}^d$. 

We provide some examples where we can directly compute the lower bounds on $\alpha$ and $\gamma$ for a restricted region $\mathcal{D}_{\epsilon}$. (See Appendix \ref{E2} for an empirical example.) Note that these bounds depend on the dataset $\{\va_i\}_{i=1}^{\infty}$, hence they are data-dependent.

\textbf{Example 1. Least Square Problem } 
Suppose that 
\begin{align}
 F(\vx) = \frac{1}{n} \sum_{i=1}^n \frac{1}{2} (\langle \va_i, \vx \rangle - y_i)^2 \label{eq:linear}  
\end{align} 
where each data point $\va_i$ consists of $d$ features and $\vy = \mA \vx^*$. Suppose the entries of all the vectors $\va_i$ are i.i.d. standard Gaussian random variables. In this case, for any fixed $\vx\in \mathbb{R}^d, \|\nabla f_i(\vx)\|^2 = \| \va_i\|^2\langle \va_i, \vx-\vx^*\rangle^2$. Let $\bar{\vx} \triangleq \frac{\vx-\vx^*}{\|\vx-\vx^*\|}$ and $Y \triangleq \langle \va_i, \bar{\vx}\rangle, \forall i $. Using the fact that a linear combination of independent normal distributions is $\mathcal{N}(\sum_j c_j\mu_j, \sum_j c_j^2 \sigma_j^2)$, $Y \sim \mathcal{N}(0, \| \bar{\vx}\|^2)$, i.e. $Y \sim \mathcal{N} (0,1)$, then $Y^2 \sim \chi^2(1)$. For example, from the distribution table of $\chi^2(1)$, $\forall i=1,2,\dots,n$,
\begin{align*}
    \mathds{P}_i\left(\|\nabla f_i(\vx)\|^2 \geq 0.45\min_j\{\|\va_j\|^2\}\|\vx-\vx^*\|^2 \right) \geq 0.5
\end{align*}
In the above case, $\alpha \geq 0.45 \min_j\{\|\va_j\|^2\}$ and $\gamma \geq 0.5$ in RUIG, where $\|\va_j\|^2 \sim \chi^2(d)$. Then, from the tail bound of $\chi^2(d)$, we have $\mathds{P}_j(\|\va_j\|^2 \geq (1-t) d ) \geq 1-e^{-dt^2/8}, \forall t\in(0,1)$. In general, $\alpha$ is not small---especially when the data is fairly dense. Furthermore, from the chi-squared distribution, other possible tuples $ \left(\frac{\alpha}{\min_j\{\|\va_j\|^2\}}, \gamma \right)$ are $\{(0.015, 0.9), (0.1, 0.75), (1.3, 0.25), (2.7, 0.1)\}$. The inequality is for any fixed $x$, so $\mathcal{D_{\epsilon}}$ is extended to $\mathcal{D}_0$. 

\textbf{Example 2. $\mu-$Strongly Convex Function} 
\begin{enumerate}[label=(\roman*)]
\item \label{ex: sc} Consider $\{f_i(\vx)\}$  $\mu_i-$strongly convex \citep{defazio2014saga} and $\vx^* = \arg\min_\vx F(\vx)$ such that $\nabla f_i(\vx^*)=\mathbf{0}$. By strong convexity, $\|\nabla f_i(\vx)\|^2 \geq (\min_j \mu_j)^2\|\vx-\vx^*\|^2, \forall \vx, \forall i=1, 2,\dots$. In this case, the uniform probability $\gamma$ degenerates to 1, $\alpha =  (\min_j\mu_j)^2$, and
$\vx$ is not restricted to $\mathcal{D}_{\epsilon}$. This class includes sum of convex functions such as squared and logistic loss with $\ell_2$-regularization.

\item A more general function class: $f_i(\vx) \in \mathcal{H}_1\cup \mathcal{H}_2$, where $\mathcal{H}_1:=\{g(\vx): g(\vx)$ is $\mu_i$-strongly convex and $\nabla g(\vx^*)=\mathbf{0}\}$ and $ \mathcal{H}_2:=\{h(\vx): h(\vx)$ is not strongly convex$\}$. $f_i(\vx)$ draws from $\mathcal{H}_1$ with probability $\gamma$ and from $\mathcal{H}_2$ with probability $1-\gamma$, where $0<\gamma<1$ and $F(\vx) = \frac{1}{n}\sum_{i=1}^n f_i(\vx)$.
\end{enumerate}

\paragraph{Convergence Under RUIG Assumption} Under the RUIG assumption, the reciprocal of the step-size (i.e. $b_t$) in AdaGrad-Norm increases quickly with high probability in Stage I until it exceeds a threshold---for example, $\eta L$---to reach Stage II.

\begin{lemma}\label{highp}
 \textbf{(Two-case high-probability lower bound for $b_N$ in the stochastic setting)} For AdaGrad-Norm (Algorithm \ref{SGDSC}), $\forall \epsilon > 0$, suppose $F(\vx)$ satisfies $(\epsilon, \alpha, \gamma)-$RUIG. For any fixed $C$, after $N = \left \lceil\frac{C^2-b_0^2}{\alpha\gamma\epsilon}+\frac{\delta}{\gamma}\right \rceil +1$ steps, either $b_N > C$ or $\min_j \|\vx_j - \vx^*\|^2 \leq \epsilon$, with high probability $1- \exp(-\frac{\delta^2}{2(N\gamma(1-\gamma)+\delta)})$.
\end{lemma}

 Letting $\delta_1 \triangleq \exp(-\frac{\delta^2}{2(N\gamma(1-\gamma)+\delta)})$, the high probability $1-\delta_1$ is derived by applying the standard Bernstein Inequality \citep{wainwright2019high} for the Bernoulli distribution (see Appendix \ref{C}). When $\gamma N/\log N \rightarrow \infty$, let $\delta = \sqrt{4c\gamma(1-\gamma) N\log N}$, then $\delta_1 \leq N^{-c}$; On the other hand, if $\gamma N \sim \log N$, let $\delta \sim (\log N)^{t+0.5}$, then $\delta_1 \sim \exp(-(c'\log N)^{2t}) \rightarrow 0$ as $N \rightarrow \infty$. As long as $\gamma \gg (\log N)^{t+0.5}/N$, the number of iterations $\frac{\delta}{\gamma} \ll N$ in Lemma \ref{highp}. In the two examples, $\gamma$ can be chosen to be at least 0.5, which leads to the high probability. In Example 2\ref{ex: sc}, $\gamma=1$, every step is deterministic, the probability degenerates to $1$.

\section{Linear Convergence Rates} \label{sec:lin}
 Throughout this section, we mainly focus on the linear convergence of AdaGrad-Norm (Algorithm \ref{SGDSC}) in both stochastic and batch settings. We highlight the robustness of the convergence rates to hyper-parameter tuning by applying our general two-stage framework. Proofs of all theorems and lemmas are in Appendix. 

\begin{theorem}\label{sgdscalg}
\textbf{(Convergence in strongly convex and stochastic  setting)}
Consider the AdaGrad-Norm Algorithm in the stochastic setting, suppose that $F(\vx)$ is strongly convex, smooth, almost stationary with $\vx^* = \arg\min_{\vx} F(\vx)$, and satisfies Restricted Uniform Inequality of Gradients (i.e. with Assumptions \ref{A1a}, \ref{A2}, \ref{A3}, \ref{A4}), then

\textbf{Case 1:} If $b_0 > \eta L$, then $\|\vx_T - \vx^*\|^2 \leq \epsilon$ with high probability $1-\delta_h$ after
$$ T =\left \lceil \frac{ b_0 + L\Delta_0 /\eta}{\mu}\log \frac{\Delta_0}{\epsilon\delta_h} \right \rceil + 1$$
iterations, where $\Delta_0 = \|\vx_0-\vx^*\|^2$;

\textbf{Case 2:} If $b_0 \leq \eta L$, then $\min_{i}\|\vx_i-\vx^*\|^2 \leq \epsilon$ with high probability $ 1-\delta_h - \exp(-\frac{\delta^2}{2(N\gamma(1-\gamma)+\delta)})$ after
$$ T =\left \lceil \frac{\eta^2L^2-b_0^2}{\alpha \gamma\epsilon}+\frac{\delta}{\gamma} + \frac{ L(\eta + \Delta/\eta) }{\mu} \log \frac{\Delta}{\epsilon\delta_h}\right \rceil + 1$$
iterations, where $\Delta = \|\vx_0-\vx^*\|^2 + \eta^2 (\log\frac{\eta^2L^2}{b_0^2} + 1)$ and $N = \left \lceil \frac{\eta^2L^2-b_0^2}{\alpha \gamma\epsilon}+\frac{\delta}{\gamma} \right \rceil$.
\end{theorem}

Our theorem establishes not only the robustness to hyper-parameters of the AdaGrad-Norm algorithm but also, more importantly, the strong linear convergence in the stochastic setting. To put the theorem in context, we compare with the sub-linear convergence rate of AdaGrad-Norm (i.e., $T = \mathcal{O}\left(1/\epsilon^2\right)$) in \cite{levy2018online, levy2017online,ward2018adagrad,li2018convergence}. The key breakthrough in our theorem is that we use a novel assumption in high dimensional probability (c.f. RUIG) and utilize the nice landscape property at the solution (c.f. Assumption \ref{A4}), instead of following the  standard analysis of SGD where it is often assumed that there is noise at the solution,  $\mathds{E}_{\xi_j}[ \|\nabla f_{\xi_{j}}(\vx) -\nabla F(\vx) \|^2] \leq \hat{\sigma}^2, \forall \vx$. 

The high probability guarantee can be verified in both stages: In Stage I, the high probability $\delta_1 \triangleq \exp(-\frac{\delta^2}{2(N\gamma(1-\gamma)+\delta)})$ is guaranteed by the high probability explanation in Lemma \ref{highp}; In Stage II, $\delta_h$ is derived from removing expectation in $\mathds{E}\|\vx-\vx^*\|^2$ with high probability $1-\delta_h$ by Markov Inequality. In general, $\delta_h$ is appropriately chosen to be a small term. 

\begin{remark}\label{remark: sto}
The classic result \citep{Needell2016} for SGD in the strongly convex setting with $\sigma^2 \triangleq \mathds{E}\|\nabla f_i(\vx^*)\|^2 = 0$ is: with stepsize $\eta_t = \frac{1}{2\sup_i L_i}$, after $T = \frac{2\sup_i L_i}{\mu}\log \frac{2\Delta_0}{\epsilon}$ iterations, $\mathds{E} \|\vx_T -\vx^*\|^2 \leq \epsilon$. Theorem \ref{sgdscalg} recovers the convergence rate up to a factor difference of $\Delta_0$ in multiplier and $\delta_h$ in the log term with high probability, if $b_0 > \sup_i L_i$. Hence, if the initialization of $\vx_0$ is extremely bad, the convergence is relatively slow. However, with tuning $\eta = \Theta(\Delta_0)$, the convergence rate is $(c_1\frac{L}{\mu}+c_2)\log\frac{\Delta_0}{\epsilon\delta_h}$ as expected. See the numerical experiments of extreme initialization of $\vx_0$ and corresponding tuning $\eta$ in Appendix \ref{E1}.
\end{remark}

In the batch setting, the full gradient at each step is available. Now, the moving direction becomes noiseless (i.e. $G_j = \nabla F(\vx_j)$), and the uniform probability $\gamma$ in $\mathds{P}_i (\|\nabla f_i(\vx)\|^2 \geq \alpha \|\vx-\vx^*\|^2) \geq \gamma$ degenerates to 1. Hence, the linear convergence rate is guaranteed in Stage II instead of with high probability.

\begin{theorem}\label{op2}
\textbf{(Convergence in strongly convex and batch setting)} Consider the AdaGrad-Norm Algorithm in the batch setting, suppose that $F(\vx)$ is $L-$smooth and $\mu$-strongly convex (i.e. with Assumptions \ref{A1a} and \ref{A2}), and $\vx^* = \arg \min_\vx F(\vx)$. Then 
$\min_{0\leq i \leq T-1}\|\vx_i-\vx^*\|^2 \leq \epsilon$
after

\textbf{Case 1:} If $b_0 > \eta \frac{\mu+L}{2}$,
$$T = 1+ \left \lceil \max\left \{ \frac{L(1+\Delta_0/\eta^2)}{\mu}, \frac{\mu+L}{2\mu} \right\} \log \frac{\Delta_0}{\epsilon} \right \rceil$$
iterations, where $\Delta_0 = \|\vx_0-\vx^*\|^2$;

\textbf{Case 2:} If $b_0\leq \eta \frac{\mu+L}{2}$,
\begin{align*}
    T = & 1 + \left \lceil 
    \max \left \{ \frac{L(1+\Delta/\eta^2)}{\mu}, \frac{\mu+L}{2\mu} \right\} \log\frac{\Delta}{\epsilon} \right. \\
    & \left. + \frac{\log(\eta^2(\mu+L)^2/4b_0^2)}{\log(1+4\mu^2\epsilon/(\mu+L)^2)} \right\rceil 
\end{align*}
iterations, where $\Delta = \|\vx_0-\vx^*\|^2 + \eta^2 (\log \frac{(\mu+L)^2}{4b_0^2} + 1)$.
\end{theorem}

\begin{remark}
Let $b_0 > \frac{\mu+L}{2}$ and $\eta = \Theta(\sqrt{\Delta_0})$, then $T = (c_1\frac{L}{\mu}+c_2)\log \frac{\Delta_0}{\epsilon}$. Theorem \ref{op2} recovers the classic result of GD with constant stepsize---whose $T = \frac{(\mu+L)^2}{4\mu L}\log\frac{\Delta_0}{\epsilon}$---up to a constant factor difference. Note that the order of $\eta$ w.r.t $\Delta_0$ is different from $\eta=\mathcal{O}(\Delta_0)$ in Remark \ref{remark: sto}, but the effect of tuning $\eta$ in both settings for extreme case is similar. 
\end{remark}

For non-convex functions that satisfy the  $\mu-$PL inequality, we extend the proof of linear convergence by bounding $F(\vx_j) -F^*$ at each step in Theorem \ref{non-convex}.

\begin{theorem}\label{non-convex}
\textbf{(Convergence in non-convex batch setting)} Consider the AdaGrad-Norm Algorithm in the batch setting, suppose that $F(\vx)$ is $L-$smooth and satisfies the $\mu-$PL inequality (i.e. with Assumptions \ref{A1b} and \ref{A2}), and $F^* = \inf_\vx F(\vx)>-\infty$, then

\textbf{Case 1:} If $b_0 > \eta L$,
$\min_{0\leq i \leq T-1}F(\vx_i)-F^* \leq \epsilon$
after
$$ T = \left \lceil \frac{b_0+\frac{2}{\eta}(F(\vx_0)-F^*)}{\mu\eta}\log\frac{F(\vx_0)-F^*}{\epsilon} \right \rceil +1$$
iterations;

\textbf{Case 2:} If $b_0 \leq \eta L$,
$\min_{0\leq i \leq T-1}F(\vx_i)-F^* \leq \epsilon$ after
$$T = \left \lceil \frac{\log(\eta^2L^2/b_0^2)}{\log(1+2\mu\epsilon/ (\eta L)^2)} + \frac{\eta L + (2/\eta)\Delta}{\mu\eta }\log\frac{\Delta}{\epsilon} \right \rceil +1 $$
iterations, where $\Delta = \frac{\eta^2L}{2}(1+2\log\frac{\eta L}{b_0})+ F(\vx_0) - F^*$. 
\end{theorem}

Compared with the result in \cite{ward2018adagrad}, our theory---using additional Assumption \ref{A1b}---significantly improves from sublinear convergence rate to linear convergence in Stage II. 
The Assumption \ref{A1b} is a generally well-known condition satisfied by
a wide range of non-convex optimization problems including over-parameterized neural networks \citep{soltanolkotabi2019theoretical, kleinberg2018alternative,li2017convergence,vaswani2018fast,wu2019global}. For the convergence of AdaGrad-Norm in the over-parameterized problem, \cite{wu2019global} proved the same convergence rate as ours. 
The convergence rate in \cite{wu2019global} was tailored to a multi-layer network with two fully connected layers. Our theorem is for general functions, however, with some additional assumptions such as $\mu-$PL inequality.

\section{Two-Stage Framework}\label{sec:proofframe}
We develop the following two-stage proof framework to analyze the
convergence rate starting from any point $x_0$ and any initial stepsize parameter $b_0$ in both the stochastic and batch settings. See the demonstration of the two-stage behavior in Figure \ref{fig:demo}. 

\textbf{Stage I} If we initialize with small $b_0$---i.e. our initial step size is large---we can get a better convergence in Stage I than SGD with constant stepsize. 
In Stage I, $b_0$ grows to some given level, such as $L$ and $\frac{\mu+L}{2}$, which depends on different settings, with deterministic iterations unless the function achieves a global minimal with tolerance $\epsilon$, i.e. $\|\vx-\vx^*\|^2\leq \epsilon $. 
Details are in two-case lemmas: Lemma \ref{highp} and \ref{twocases}. 
By Lemma \ref{boundset}, $\|\vx-\vx^*\|$ is bounded by radius $\Delta = \mathcal{R}(b_0, \|\vx_0-\vx^*\|, C)$ before $b_t$ grows up to $C$, instead of blowing up.

\textbf{Stage II} After Stage I, $b_t$ exceeds a certain threshold deterministically in the batch setting and with high probability in the stochastic setting. Conditioned on this, the update is a contraction in the strongly convex setting, i.e. $\|\vx_{j+1} - \vx^*\|^2 \leq (1 - \mathcal{P}(b_{\max}, \mu, L))\|\vx_j - \vx^*\|^2$, where $\mathcal{P}$ is a function s.t. $0 < \mathcal{P}(b_{\max},\mu,L) < 1 $. $b_{\max}$ is bounded by Lemma \ref{bmax}. 

\begin{figure*}[t]
    \centering
    \includegraphics[width=0.9\textwidth]{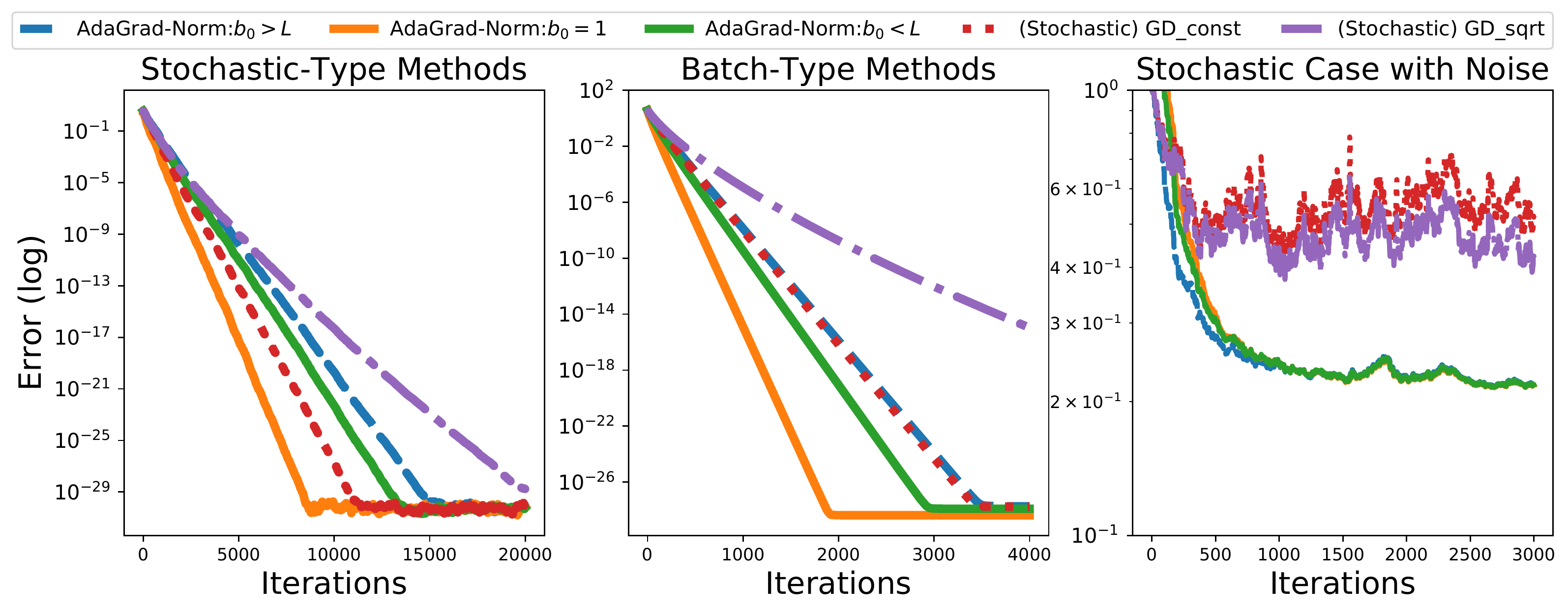}
    \caption{Error in log scale of least square problem: $F(\vx)= \frac{1}{2n} \|\mA \vx - \vy\|^2$. The left and central figures show error $\|\vx_t-\vx^*\|^2$ in the noiseless case. The right figure shows the loss $F(\vx_t)$ in the noisy case.}
    \label{fig:ls}
\end{figure*}

\subsection{Growth of $b_t$ in Stage I}
We introduce some lemmas that are critical in the proof of the growth of $b_t$ in Stage I in the section. Note that in the stochastic setting, RUIG is a sufficient condition for $b_t$'s growth in Stage I to achieve a certain threshold with high probability, so the corresponding two-case growth of $b_t$ is provided in Lemma \ref{highp}. Detailed proofs are provided in Appendix \ref{C}.

\begin{lemma}\label{twocases}
\textbf{(Two-case lower bound for $b_N$ in the batch setting)}
For fixed $\epsilon \in (0,1)$ and $C$, consider AdaGrad-Norm in the batch setting to minimize the  objective function $F(\vx)$, then
\begin{enumerate}[label=(\alph*)]
    \item If $F(\vx)$ is $\mu-$strongly convex, then after $N = \left \lceil \frac{\log(C^2/b_0^2)}{\log(1+\mu^2\epsilon/ C^2)} \right \rceil + 1$ iterations, either $ b_N > C $ or $\min_{0\leq i \leq N-1} \|\vx_i - \vx^*\|^2 \leq \epsilon$, where $\vx^* = \arg\min F(\vx)$;
    
    \item If $F(\vx)$ is a non-convex function satisfying $\mu-$PL inequality, then after $N = \left \lceil \frac{\log(C^2/b_0^2)}{\log(1+2\mu\epsilon/ C^2)} \right \rceil$ iterations, either $b_N>C$ or $\min_{0\leq i \leq N-1} F(\vx_i) - F^* \leq \epsilon$, where $F^* = \inf_\vx F(\vx) > -\infty$. 
\end{enumerate}
\end{lemma}

Lemma \ref{highp} and \ref{twocases} depict the two-stage growth of $b_t$, which is less and over certain thresholds, in stochastic and batch settings, respectively.
\begin{remark}
In Lemma \ref{highp} and \ref{twocases}, we provide the worst cases for the growth of $b_t$. 
However, $b_t$ actually grows very quickly in practice, especially in the stochastic setting. 
For Lemma \ref{twocases}, since $\log(1+x) \sim x$, for $x = \mu^2\epsilon/C^2$ small, $N \sim \frac{C^2}{\mu^2\epsilon} \log\frac{C^2}{b_0}$.
\end{remark}

\begin{lemma}\label{boundset}
\textbf{(Upper bound for $\|\vx_{J-1} - \vx^*\|^2 $)} 
For any fixed $C$ and $\eta$, consider AdaGrad-Norm in either stochastic or batch setting with $ G_j(\vx^*) = 0, \forall j$ (stochastic: $\nabla f_j(\vx^*)=\mathbf{0}$; batch: $\nabla F(\vx^*) =\mathbf{0}$) using update rule $b_{j+1}^2 = b_j^2 + \|G_j\|^2$. Suppose that $J$ is the first index s.t. $ b_J > C$, then $$\|\vx_{J-1} - \vx^*\|^2 \leq \|\vx_0-\vx^*\|^2 + \eta^2 (\log (C^2/b_0^2) + 1)$$
\end{lemma}
Lemma \ref{boundset} gives an upper bound on the distance between the snapshot before contraction and the optimal solution $\vx^*$ i.e. $\|\vx_{J-1} - \vx^*\|$. It guarantees that the extreme distance to $\vx^*$ is always bounded during AdaGrad-Norm updates, even without projection, or additional assumption, for example, $ \forall t, \|\vx_t-\vx^*\|_2\leq D$, for some constant $D$, in Adam \citep{kingma2014adam} and AMSGrad \citep{j.2018on}.

\subsection{Upper Bounds on $b_t$ in Stage II}
In Stage II, we focus on the maximum value that $b_t$ can obtain during the optimization process.
\begin{lemma}\label{bmax}
\textbf{(Upper bound for $b_{\max}$)}
Consider AdaGrad-Norm in either stochastic or batch setting with $ G_j(\vx^*) =\mathbf{0}, \forall j$ (stochastic: $\nabla f_j(\vx^*)=\mathbf{0}$; batch: $\nabla F(\vx^*) =\mathbf{0}$), for any fixed $C\geq \eta L$, if $J$ is the first index s.t. $ b_J > C$, then $b_{\max} \triangleq \max_{l\geq 0}b_{J+l}$ is upper bounded by $$b_{\max}\leq  C + (L/\eta) (\|\vx_0-\vx^*\|^2 + \eta^2 (\log (C^2/b_0^2) + 1))$$
\end{lemma}

Lemma \ref{bmax} indicates that even though $b_t^2$ increases due to adding $\|G_t\|^2$ to $b_t^2$ at each iteration, it is always upper bounded by $b_{\max}$. The asymptotic behavior of the stepsize (i.e. $\frac{\eta}{b_t}$) is $\mathcal{O}(\frac{1}{\sqrt{t}})$ at first, and it approaches to a constant in the end as $\vx_t \rightarrow \vx^*$, which also explains the auto-tuning nature of AdaGrad-Norm.

After $b_t$ exceeds certain thresholds like $\eta L/2$, the following Lemma \ref{descent} shows that AdaGrad-Norm is indeed a descent algorithm, i.e. $\|\vx_t-\vx^*\|^2$ will not increase subsequently, so we can take $\vx_T$ as $\vx_{best}$ in Stage II.

\begin{lemma}\label{descent}
\textbf{(Descent lemma for $\|\vx_t-\vx^*\|^2$)}
Once $b_j > \eta L/2$, Algorithm \ref{SGDSC} is a descent algorithm for the error $\|\vx_t - \vx^*\|^2$. Furthermore, if $\|\vx_{j-1} -\vx^*\|^2 \leq \Delta$, then $\forall l\geq 0$, $x_{j-1+l}$ will stay in the ball centering at $\vx^*$ with radius $\sqrt{\Delta}$, i.e. $\|\vx_{j-1+l}-\vx^*\|^2 \leq \Delta$.
\end{lemma}
\vspace{-0.3cm}
\begin{figure*}[ht]
     \centering
     \begin{minipage}{0.32\textwidth}
         \centering
         \includegraphics[width=\textwidth]{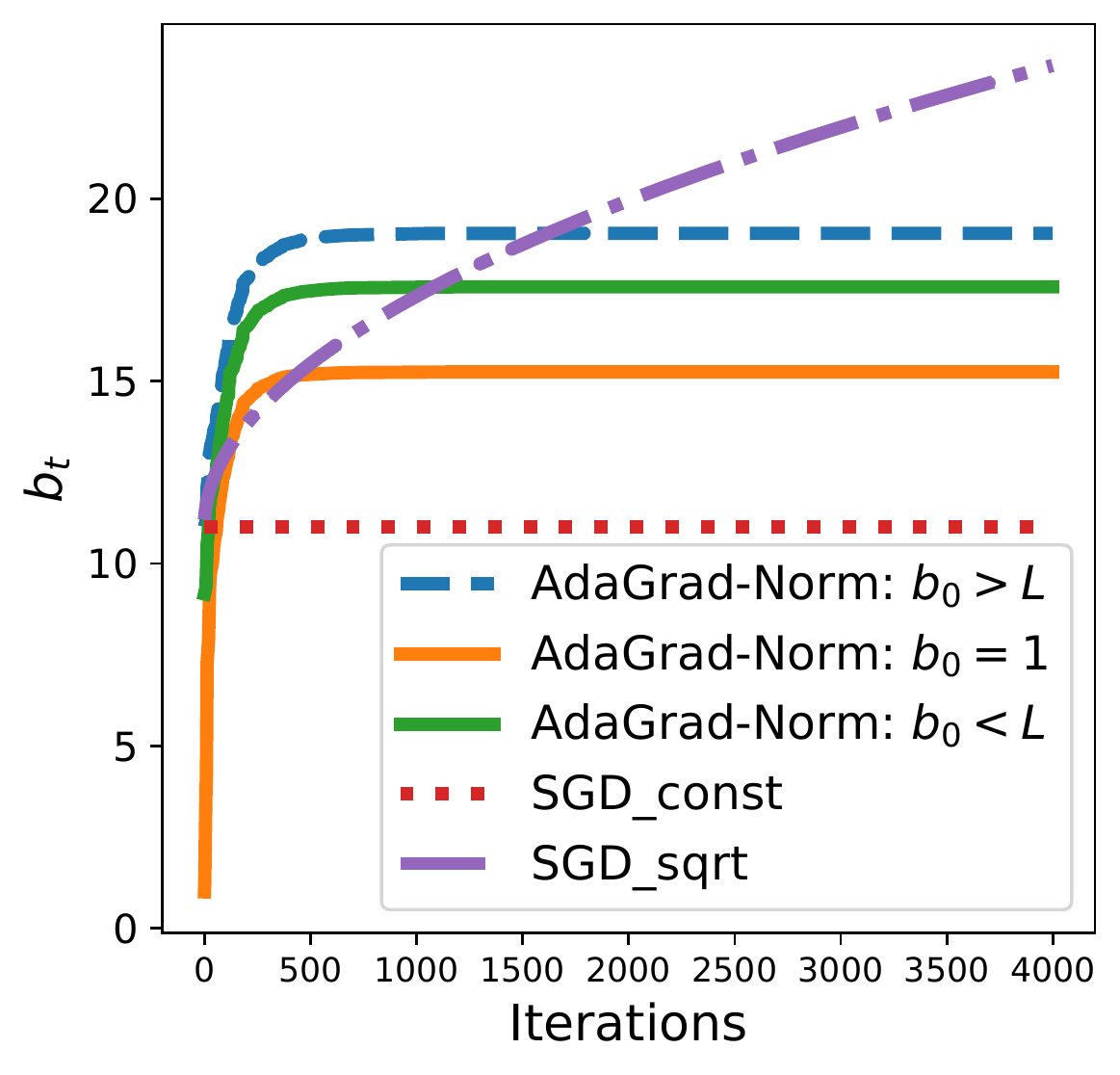}
         \caption{Growth of $b_t$ in stochastic setting using different algorithms: AdaGrad-Norm and SGD.}
         \label{fig:bgrow}
     \end{minipage} 
     \hfill
     \begin{minipage}{0.33\textwidth}
         \centering
         \includegraphics[width=\textwidth]{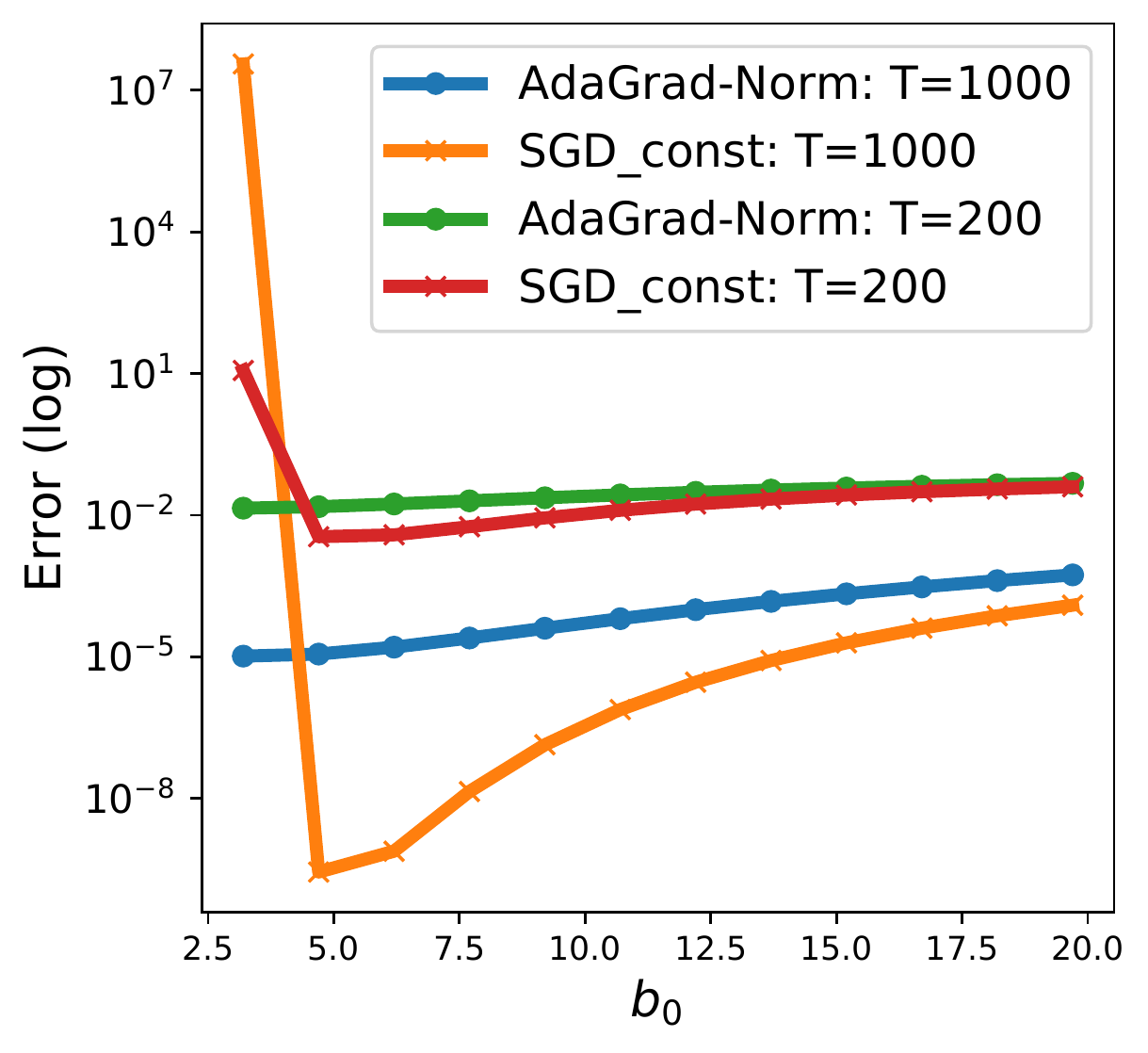}
         \caption{Robustness of AdaGrad-Norm: $\|\vx_T-\vx^*\|^2$ in log scale with different choices of $b_0$.}
         \label{fig:anyb0}
     \end{minipage}
       \hfill
     \begin{minipage}{0.33\textwidth}
         \centering
         \includegraphics[width=\textwidth]{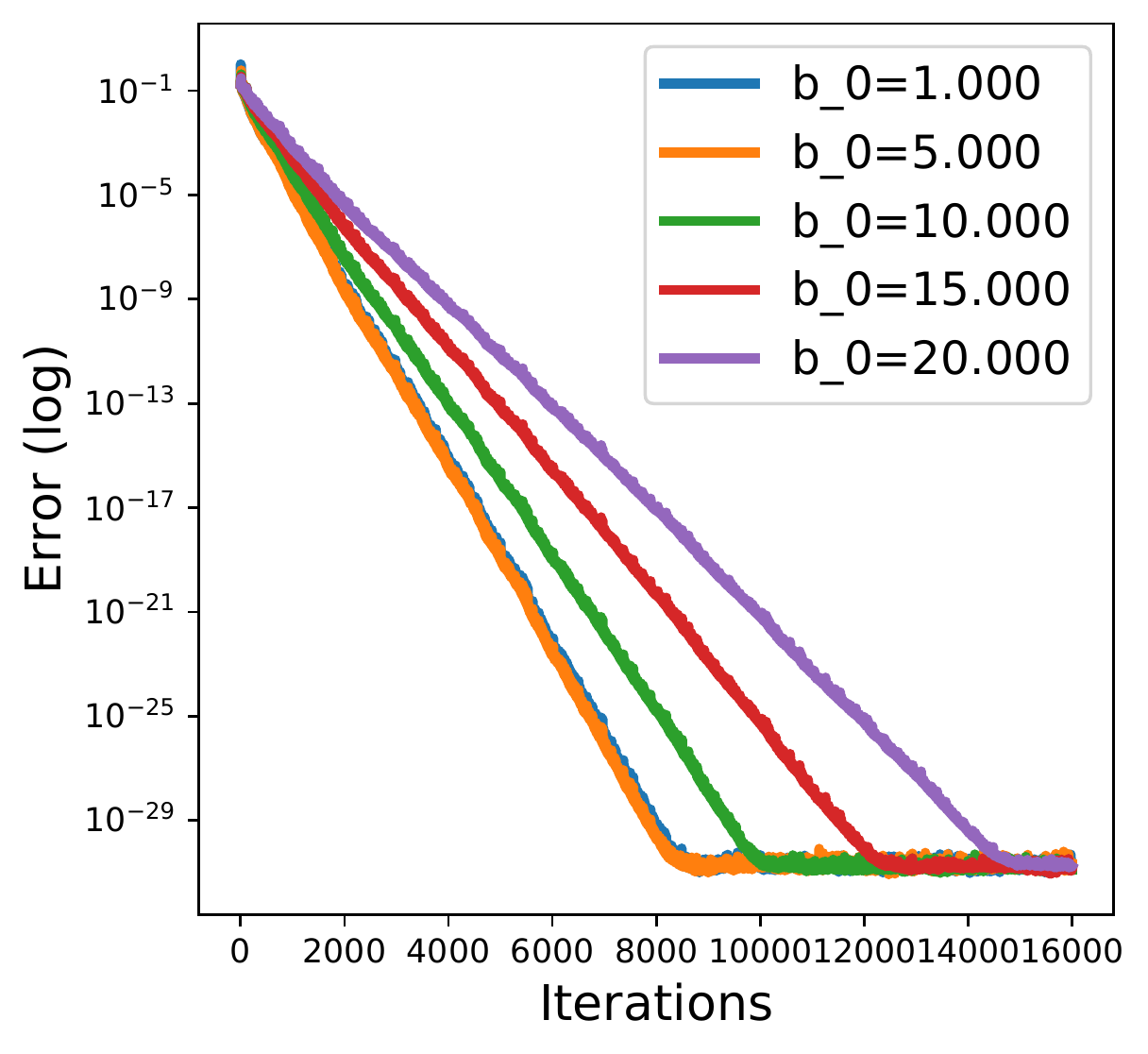}
         \caption{Comparison of different initial stepsizes ($\eta_0 = \frac{\eta}{b_0}$) of AdaGrad-Norm with $\|\vx_t-\vx^*\|^2$ in log scale.}
         \label{fig:b0linear}
     \end{minipage}
\end{figure*}

\section{Numerical Experiments} \label{sec:exp}
In this section, we present numerical results to compare AdaGrad-Norm and (stochastic) Gradient Descent methods with fixed stepsize $\eta_j = \frac{1}{b_0}$ (GD\_const or SGD\_const) or square-root decaying stepsize $\eta_j = \frac{1}{b_0 + 0.2\sqrt{j}}$ (GD\_sqrt or SGD\_sqrt), in the stochastic and batch settings, respectively.

Consider the least square problem from (\ref{eq:linear}).
The Lipschitz constants are $L_i = \|\va_i\|^2$ and $\bar{L} = \sum_{i=1}^n \frac{1}{n}\|\va_i\|^2 = \frac{1}{n}\|\mA\|_F^2$, respectively. In the experiments, we setup the noiseless problem with a $1000 \times 20$ random matrix $\mA$ and a vector $\vx^*$ with $\vy = \mA \vx^*$.

We first illustrate the linear convergence and robustness in the noiseless cases. Figure \ref{fig:ls} verifies the expected linear convergence of AdaGrad-Norm in the stochastic and batch settings. In order to compare the convergence rates of AdaGrad-Norm with vanilla (S)GD, we choose $\eta=1$ and a $b_0 > \sup_i L_i \triangleq L$ to prevent (S)GD from blowing up. AdaGrad-Norm with $b_0 = 1$, $b_0<L$ and $b_0 > L$ have similar linear convergence as (S)GD\_const, up to a constant difference, while (S)GD\_sqrt converges more slowly.

Figure \ref{fig:ls} shows that even simply setting $b_0 = 1$, AdaGrad-Norm has a better convergence rate than that of non-adaptive (S)GD, since AdaGrad-Norm takes a big stepsize when $\vx$ is far away from $\vx^*$, and then very small stepsize around $\vx^*$ when $b_j$ grows to a value $b_{\max}$. Eventually, $b_j$ converges to a constant value since $\|G_j\|\rightarrow0$ or $\|\nabla F(\vx_j)\|\rightarrow 0$ as $\vx_j \rightarrow \vx^*$. In the noisy case (Figure \ref{fig:ls} Right), AdaGrad-Norm has a similar convergence rate up to a constant factor and achieves a better approximation of $\vx^*$, with less vibrations compared to SGD\_const or SGD\_sqrt.

Figure \ref{fig:bgrow} shows that the growth of $b_t$ in AdaGrad-Norm is similar to SGD\_sqrt at first, but after exceeding the threshold and approximately reaching $b_{\max}$, $b_t$'s growth is similar to SGD\_const. Figure \ref{fig:anyb0} and \ref{fig:b0linear} show that the linear convergence rates of AdaGrad-Norm are more robust to the choice of initial stepsize $1/b_0$ compared to SGD\_const. The error $\|\vx_T-\vx^*\|^2$ of AdaGrad-Norm after $T$ iterations remains stable for a relatively arbitrary range of $b_0$ while the error of SGD\_const blows up at first and then decreases significantly when $b_0$ approaches to $L$ since SGD\_const is sensitive to the choice of stepsize.

The result of experiment on one hidden layer over-parameterized neural net and experimental details are in Appendix \ref{E2}. Figure \ref{fig:over} shows that (1) AdaGrad-Norm converges faster than GD\_const and almost linearly; (2) The gradients of the first few iterations are often big enough to accumulate to exceed $\eta L$, which empirically verifies Assumption \ref{A3}.

\section{Discussions} \label{sec:diss}
In this work, we propose the notion of RUIG to measure the uniform lower bound of gradients with respect to $\|\vx-\vx^*\|^2$ in a restricted region. We propose a two-stage framework and use it to prove the non-asymptotic convergence rates for AdaGrad-Norm starting from any initialization and without knowing the smooth or strongly convex parameter a priori. In the stochastic setting, we prove linear convergence with high probability under strongly convex and RUIG assumptions, without requiring a uniform bound on $\mathds{E}\|G_t\|^2$. In the batch setting, we prove deterministic linear convergence for strongly convex functions and non-convex functions with PL inequality. Both theoretical and numerical results validate the robustness of AdaGrad-Norm starting at any initial stepsize.

There are still some open problems to be solved: First, drawing on \cite{Needell2016}, we may improve $L = \sup_i L_i$ in convergence rates to $\bar{L} =\frac{1}{n}\sum_i L_i$ with importance sampling. Second, extending Assumption \ref{A4} to the weak growth condition $\mathds{E}_{\xi_t}[\|\nabla f_{\xi_t}(\vx_t)\|^2]\leq M\|\nabla F(\vx_t)\|^2 + \sigma^2$ in \cite{Cevher2019} may lead to a more general result. 
Third, since AdaGrad-Norm is fundamentally related to both Adam and AMSGrad, extending our theoretical guarantees to the two algorithms is an exciting direction for future research.

\section*{Acknowledgments}
We thank anonymous reviewers, Amelia Henriksen, Jiayi Wei, and Thomas Herben for helpful comments, which greatly improved the manuscript. We thank Purnamrita Sarkar for helpful discussion. This project was supported in part by AFOSR MURI Award N00014-17-S-F006. 

\bibliographystyle{unsrtnat}
\bibliography{adagrad}

\begin{thebibliography}{48}
\providecommand{\natexlab}[1]{#1}
\providecommand{\url}[1]{\texttt{#1}}
\expandafter\ifx\csname urlstyle\endcsname\relax
  \providecommand{\doi}[1]{doi: #1}\else
  \providecommand{\doi}{doi: \begingroup \urlstyle{rm}\Url}\fi

\bibitem[Lojasiewicz(1963)]{lojasiewicz1963propriete}
Stanislaw Lojasiewicz.
\newblock Une propri{\'e}t{\'e} topologique des sous-ensembles analytiques
  r{\'e}els.
\newblock \emph{Les {\'e}quations aux d{\'e}riv{\'e}es partielles},
  117:\penalty0 87--89, 1963.

\bibitem[Polyak(1963)]{polyak1963gradient}
Boris~Teodorovich Polyak.
\newblock Gradient methods for minimizing functionals.
\newblock \emph{Zhurnal Vychislitel'noi Matematiki i Matematicheskoi Fiziki},
  3\penalty0 (4):\penalty0 643--653, 1963.

\bibitem[Bottou and Cun(2004)]{bottou2004large}
L{\'e}on Bottou and Yann~L Cun.
\newblock Large scale online learning.
\newblock In \emph{Advances in neural information processing systems}, pages
  217--224, 2004.

\bibitem[Bottou et~al.(2018)Bottou, Curtis, and
  Nocedal]{bottou2018optimization}
L{\'e}on Bottou, Frank~E Curtis, and Jorge Nocedal.
\newblock Optimization methods for large-scale machine learning.
\newblock \emph{Siam Review}, 60\penalty0 (2):\penalty0 223--311, 2018.

\bibitem[Bottou(1991)]{bottou-91a}
{L\'eon} Bottou.
\newblock \emph{Une Approche th\'eorique de l'Apprentissage Connexionniste:
  Applications \`a la Reconnaissance de la Parole}.
\newblock PhD thesis, Universit\'{e} de Paris XI, Orsay, France, 1991.

\bibitem[Nash and Nocedal(1991)]{nash1991numerical}
Stephen~G Nash and Jorge Nocedal.
\newblock A numerical study of the limited memory bfgs method and the
  truncated-newton method for large scale optimization.
\newblock \emph{SIAM Journal on Optimization}, 1\penalty0 (3):\penalty0
  358--372, 1991.

\bibitem[Bertsekas(1999)]{bertsekas1999nonlinear}
Dimitri~P Bertsekas.
\newblock \emph{Nonlinear programming}.
\newblock Athena Scientific, 1999.

\bibitem[Nesterov(2005)]{nesterov2005smooth}
Yu~Nesterov.
\newblock Smooth minimization of non-smooth functions.
\newblock \emph{Mathematical programming}, 103\penalty0 (1):\penalty0 127--152,
  2005.

\bibitem[Haykin et~al.(2005)]{haykin2005cognitive}
Simon Haykin et~al.
\newblock Cognitive radio: brain-empowered wireless communications.
\newblock \emph{IEEE journal on selected areas in communications}, 23\penalty0
  (2):\penalty0 201--220, 2005.

\bibitem[Bubeck et~al.(2015)]{bubeck2015convex}
S{\'e}bastien Bubeck et~al.
\newblock Convex optimization: Algorithms and complexity.
\newblock \emph{Foundations and Trends{\textregistered} in Machine Learning},
  8\penalty0 (3-4):\penalty0 231--357, 2015.

\bibitem[Allen-Zhu et~al.(2018)Allen-Zhu, Li, and Song]{allen2018convergence}
Zeyuan Allen-Zhu, Yuanzhi Li, and Zhao Song.
\newblock A convergence theory for deep learning via over-parameterization.
\newblock \emph{arXiv preprint arXiv:1811.03962}, 2018.

\bibitem[Zou et~al.(2018{\natexlab{a}})Zou, Cao, Zhou, and
  Gu]{zou2018stochastic}
Difan Zou, Yuan Cao, Dongruo Zhou, and Quanquan Gu.
\newblock Stochastic gradient descent optimizes over-parameterized deep relu
  networks.
\newblock \emph{arXiv preprint arXiv:1811.08888}, 2018{\natexlab{a}}.

\bibitem[Moulines and Bach(2011)]{moulines2011non}
Eric Moulines and Francis~R Bach.
\newblock Non-asymptotic analysis of stochastic approximation algorithms for
  machine learning.
\newblock In \emph{Advances in Neural Information Processing Systems}, pages
  451--459, 2011.

\bibitem[Needell et~al.(2016)Needell, Srebro, and Ward]{Needell2016}
Deanna Needell, Nathan Srebro, and Rachel Ward.
\newblock Stochastic gradient descent, weighted sampling, and the randomized
  kaczmarz algorithm.
\newblock \emph{Mathematical Programming}, 155\penalty0 (1):\penalty0 549--573,
  Jan 2016.
\newblock ISSN 1436-4646.

\bibitem[Schmidt et~al.(2017)Schmidt, Le~Roux, and Bach]{schmidt2017minimizing}
Mark Schmidt, Nicolas Le~Roux, and Francis Bach.
\newblock Minimizing finite sums with the stochastic average gradient.
\newblock \emph{Mathematical Programming}, 162\penalty0 (1-2):\penalty0
  83--112, 2017.

\bibitem[Johnson and Zhang(2013)]{johnson2013accelerating}
Rie Johnson and Tong Zhang.
\newblock Accelerating stochastic gradient descent using predictive variance
  reduction.
\newblock In \emph{Advances in neural information processing systems}, pages
  315--323, 2013.

\bibitem[Defazio et~al.(2014)Defazio, Bach, and
  Lacoste-Julien]{defazio2014saga}
Aaron Defazio, Francis Bach, and Simon Lacoste-Julien.
\newblock Saga: A fast incremental gradient method with support for
  non-strongly convex composite objectives.
\newblock In \emph{Advances in neural information processing systems}, pages
  1646--1654, 2014.

\bibitem[Duchi et~al.(2011)Duchi, Hazan, and Singer]{duchi2011adaptive}
John Duchi, Elad Hazan, and Yoram Singer.
\newblock Adaptive subgradient methods for online learning and stochastic
  optimization.
\newblock \emph{Journal of Machine Learning Research}, 12\penalty0
  (Jul):\penalty0 2121--2159, 2011.

\bibitem[McMahan and Streeter(2010)]{mcmahan2010adaptive}
H~Brendan McMahan and Matthew Streeter.
\newblock Adaptive bound optimization for online convex optimization.
\newblock \emph{arXiv preprint arXiv:1002.4908}, 2010.

\bibitem[Kingma and Ba(2014)]{kingma2014adam}
Diederik~P Kingma and Jimmy Ba.
\newblock Adam: A method for stochastic optimization.
\newblock \emph{arXiv preprint arXiv:1412.6980}, 2014.

\bibitem[Lafond et~al.(2017)Lafond, Vasilache, and
  Bottou]{lafond-vasilache-bottou-2017}
Jean Lafond, Nicolas Vasilache, and L\'{e}on Bottou.
\newblock Diagonal rescaling for neural networks.
\newblock Technical report, arXiV:1705.09319, 2017.

\bibitem[Reddi et~al.(2018{\natexlab{a}})Reddi, Kale, and Kumar]{Reddi2018OnTC}
Sashank~J. Reddi, Satyen Kale, and Sanjiv Kumar.
\newblock On the convergence of adam and beyond.
\newblock \emph{CoRR}, abs/1904.09237, 2018{\natexlab{a}}.

\bibitem[Shah et~al.(2018)Shah, Kyrillidis, and Sanghavi]{shah2018minimum}
Vatsal Shah, Anastasios Kyrillidis, and Sujay Sanghavi.
\newblock Minimum norm solutions do not always generalize well for
  over-parameterized problems.
\newblock \emph{arXiv preprint arXiv:1811.07055}, 2018.

\bibitem[Zou et~al.(2018{\natexlab{b}})Zou, Shen, Jie, Zhang, and
  Liu]{zou2018sufficient}
Fangyu Zou, Li~Shen, Zequn Jie, Weizhong Zhang, and Wei Liu.
\newblock A sufficient condition for convergences of adam and rmsprop.
\newblock \emph{arXiv preprint arXiv:1811.09358}, 2018{\natexlab{b}}.

\bibitem[Staib et~al.(2019)Staib, Reddi, Kale, Kumar, and
  Sra]{staib2019escaping}
Matthew Staib, Sashank~J Reddi, Satyen Kale, Sanjiv Kumar, and Suvrit Sra.
\newblock Escaping saddle points with adaptive gradient methods.
\newblock \emph{arXiv preprint arXiv:1901.09149}, 2019.

\bibitem[Levy(2017)]{levy2017online}
Kfir Levy.
\newblock Online to offline conversions, universality and adaptive minibatch
  sizes.
\newblock In \emph{Advances in Neural Information Processing Systems}, pages
  1613--1622, 2017.

\bibitem[Ward et~al.(2018)Ward, Wu, and Bottou]{ward2018adagrad}
Rachel Ward, Xiaoxia Wu, and Leon Bottou.
\newblock Adagrad stepsizes: Sharp convergence over nonconvex landscapes, from
  any initialization.
\newblock \emph{arXiv preprint arXiv:1806.01811}, 2018.

\bibitem[Wu et~al.(2018)Wu, Ward, and Bottou]{wu2018wngrad}
Xiaoxia Wu, Rachel Ward, and L{\'e}on Bottou.
\newblock Wngrad: learn the learning rate in gradient descent.
\newblock \emph{arXiv preprint arXiv:1803.02865}, 2018.

\bibitem[Levy et~al.(2018)Levy, Yurtsever, and Cevher]{levy2018online}
Yehuda~Kfir Levy, Alp Yurtsever, and Volkan Cevher.
\newblock Online adaptive methods, universality and acceleration.
\newblock In S.~Bengio, H.~Wallach, H.~Larochelle, K.~Grauman, N.~Cesa-Bianchi,
  and R.~Garnett, editors, \emph{Advances in Neural Information Processing
  Systems 31}, pages 6500--6509. Curran Associates, Inc., 2018.

\bibitem[Li and Orabona(2018)]{li2018convergence}
Xiaoyu Li and Francesco Orabona.
\newblock On the convergence of stochastic gradient descent with adaptive
  stepsizes.
\newblock \emph{arXiv preprint arXiv:1805.08114}, 2018.

\bibitem[Balles and Hennig(2018)]{balles2018dissecting}
Lukas Balles and Philipp Hennig.
\newblock Dissecting adam: The sign, magnitude and variance of stochastic
  gradients.
\newblock In \emph{International Conference on Machine Learning}, pages
  413--422, 2018.

\bibitem[Bernstein et~al.(2018)Bernstein, Wang, Azizzadenesheli, and
  Anandkumar]{bernstein_signum}
Jeremy Bernstein, Yu-Xiang Wang, Kamyar Azizzadenesheli, and Animashree
  Anandkumar.
\newblock sign{SGD}: {C}ompressed {O}ptimisation for {N}on-{C}onvex {P}roblems.
\newblock In \emph{International Conference on Machine Learning (ICML-18)},
  2018.

\bibitem[Mukkamala and Hein(2017)]{mukkamala2017variants}
Mahesh~Chandra Mukkamala and Matthias Hein.
\newblock Variants of rmsprop and adagrad with logarithmic regret bounds.
\newblock In \emph{Proceedings of the 34th International Conference on Machine
  Learning-Volume 70}, pages 2545--2553. JMLR. org, 2017.

\bibitem[Chen et~al.(2018)Chen, Xu, Chen, and Yang]{pmlr-v80-chen18m}
Zaiyi Chen, Yi~Xu, Enhong Chen, and Tianbao Yang.
\newblock {SADAGRAD}: Strongly adaptive stochastic gradient methods.
\newblock In Jennifer Dy and Andreas Krause, editors, \emph{Proceedings of the
  35th International Conference on Machine Learning}, volume~80 of
  \emph{Proceedings of Machine Learning Research}, pages 913--921,
  Stockholmsmässan, Stockholm Sweden, 10--15 Jul 2018. PMLR.

\bibitem[Vaswani et~al.(2018)Vaswani, Bach, and Schmidt]{vaswani2018fast}
Sharan Vaswani, Francis Bach, and Mark Schmidt.
\newblock Fast and faster convergence of sgd for over-parameterized models and
  an accelerated perceptron.
\newblock \emph{arXiv preprint arXiv:1810.07288}, 2018.

\bibitem[Zhang et~al.(2016)Zhang, Bengio, Hardt, Recht, and
  Vinyals]{zhang2016understanding}
Chiyuan Zhang, Samy Bengio, Moritz Hardt, Benjamin Recht, and Oriol Vinyals.
\newblock Understanding deep learning requires rethinking generalization.
\newblock \emph{arXiv preprint arXiv:1611.03530}, 2016.

\bibitem[Du et~al.(2019)Du, Zhai, Poczos, and Singh]{du2018gradient}
Simon~S. Du, Xiyu Zhai, Barnabas Poczos, and Aarti Singh.
\newblock Gradient descent provably optimizes over-parameterized neural
  networks.
\newblock In \emph{International Conference on Learning Representations}, 2019.

\bibitem[Zhou et~al.(2019)Zhou, Yang, Zhang, Liang, and Tarokh]{zhou2018sgd}
Yi~Zhou, Junjie Yang, Huishuai Zhang, Yingbin Liang, and Vahid Tarokh.
\newblock {SGD} converges to global minimum in deep learning via star-convex
  path.
\newblock In \emph{International Conference on Learning Representations}, 2019.

\bibitem[Bassily et~al.(2018)Bassily, Belkin, and Ma]{bassily2018exponential}
Raef Bassily, Mikhail Belkin, and Siyuan Ma.
\newblock On exponential convergence of sgd in non-convex over-parametrized
  learning.
\newblock \emph{arXiv preprint arXiv:1811.02564}, 2018.

\bibitem[Roux et~al.(2012)Roux, Schmidt, and Bach]{bach2012a}
Nicolas~L. Roux, Mark Schmidt, and Francis~R. Bach.
\newblock A stochastic gradient method with an exponential convergence \_rate
  for finite training sets.
\newblock In F.~Pereira, C.~J.~C. Burges, L.~Bottou, and K.~Q. Weinberger,
  editors, \emph{Advances in Neural Information Processing Systems 25}, pages
  2663--2671. Curran Associates, Inc., 2012.

\bibitem[Cevher and V{\~{u}}(2019)]{Cevher2019}
Volkan Cevher and Báº±ng~C{\^o}ng V{\~{u}}.
\newblock On the linear convergence of the stochastic gradient method with
  constant step-size.
\newblock \emph{Optimization Letters}, 13\penalty0 (5):\penalty0 1177--1187,
  Jul 2019.
\newblock ISSN 1862-4480.

\bibitem[Schmidt and Roux(2013)]{schmidt2013fast}
Mark Schmidt and Nicolas~Le Roux.
\newblock Fast convergence of stochastic gradient descent under a strong growth
  condition.
\newblock \emph{arXiv preprint arXiv:1308.6370}, 2013.

\bibitem[Wainwright(2019)]{wainwright2019high}
Martin~J Wainwright.
\newblock \emph{High-dimensional statistics: A non-asymptotic viewpoint},
  volume~48.
\newblock Cambridge University Press, 2019.

\bibitem[Soltanolkotabi et~al.(2019)Soltanolkotabi, Javanmard, and
  Lee]{soltanolkotabi2019theoretical}
Mahdi Soltanolkotabi, Adel Javanmard, and Jason~D Lee.
\newblock Theoretical insights into the optimization landscape of
  over-parameterized shallow neural networks.
\newblock \emph{IEEE Transactions on Information Theory}, 65\penalty0
  (2):\penalty0 742--769, 2019.

\bibitem[Kleinberg et~al.(2018)Kleinberg, Li, and
  Yuan]{kleinberg2018alternative}
Robert Kleinberg, Yuanzhi Li, and Yang Yuan.
\newblock An alternative view: When does sgd escape local minima?
\newblock In \emph{International Conference on Machine Learning}, pages
  2703--2712, 2018.

\bibitem[Li and Yuan(2017)]{li2017convergence}
Yuanzhi Li and Yang Yuan.
\newblock Convergence analysis of two-layer neural networks with relu
  activation.
\newblock In \emph{Advances in Neural Information Processing Systems}, pages
  597--607, 2017.

\bibitem[Wu et~al.(2019)Wu, Du, and Ward]{wu2019global}
Xiaoxia Wu, Simon~S Du, and Rachel Ward.
\newblock Global convergence of adaptive gradient methods for an
  over-parameterized neural network.
\newblock \emph{arXiv preprint arXiv:1902.07111}, 2019.

\bibitem[Reddi et~al.(2018{\natexlab{b}})Reddi, Kale, and Kumar]{j.2018on}
Sashank~J. Reddi, Satyen Kale, and Sanjiv Kumar.
\newblock On the convergence of adam and beyond.
\newblock In \emph{International Conference on Learning Representations},
  2018{\natexlab{b}}.

\end{thebibliography}
\onecolumn
\appendix

\setcounter{lemma}{0}
\setcounter{equation}{0}
\renewcommand{\thelemma}{\Alph{section}\arabic{lemma}}
\section*{Supplementary Material}
Supplementary material for the paper: "Linear Convergence of Adaptive Stochastic Gradient Descent".

This appendix is organized as follows:
\begin{itemize}
    \item Appendix \ref{A}: Proof of Theorem \ref{sgdscalg} in the Stochastic Setting 
    \item Appendix \ref{B}: Proof of Theorem \ref{op2} and \ref{non-convex} in the batch Setting
    \item Appendix \ref{C}: Proof of Lemmas in Stage I
    \item Appendix \ref{D}: Proof of Lemmas in Stage II
    \item Appendix \ref{E}: More Numerical Experiments
\end{itemize}

\section{Proof of Theorem \ref{sgdscalg} in the Stochastic Setting}\label{A}
From Lemma \ref{highp}, let $C=\eta L$, after $N \geq \frac{\eta^2L^2-b_0^2}{\alpha\gamma\epsilon}+\frac{\delta}{\gamma}$ steps, 
if $\min_{0\leq i\leq N-1}\|\vx_i-\vx^*\|^2 > \epsilon$, then with high probability $ 1- \exp(-\frac{\delta^2}{2(N\gamma(1-\gamma)+\delta)})$, $b_N > \eta L $. Then, there exists a first index $k_0 < N$, s.t. $b_{k_0} >\eta L$ but $b_{k_0-1} < \eta L$.\\
If $k_0 \geq 1$, then
\begin{equation}
\begin{split}
     \|\vx_{k_0+l}-\vx^*\|^2  & = \|\vx_{k_0-1+l}-\vx^*\|^2 + \frac{\eta^2}{b_{k_0+l}^2}\|G_{k_0-1+l}\|^2 - \frac{2\eta}{b_{k_0+l}}\langle \vx_{k_0-1+l}-\vx^*, G_{k_0-1+l}\rangle \\
     & \leq \|\vx_{k_0-1+l}-\vx^*\|^2 + (\frac{\eta^2L}{b_{k_0+l}^2}- \frac{2\eta}{b_{k_0+l}}) \langle \vx_{k_0-1+l}-\vx^*, G_{k_0-1+l}\rangle\\
     & \leq \|\vx_{k_0-1+l}-\vx^*\|^2 - \frac{\eta}{b_{k_0+l}}\langle \vx_{k_0-1+l} - \vx^*, G_{k_0-1+l}\rangle\\
     & \leq \|\vx_{k_0-1+l}-\vx^*\|^2 - \frac{\eta}{b_{\max}}\langle \vx_{k_0-1+l}-\vx^*, G_{k_0-1+l}\rangle\\
\end{split}
\end{equation}
where the last second inequality is from the condition $b_{k_0} > \eta L$. The last inequality holds since $b_{k_0+l} \leq b_{\max}$ and $f_{\xi_{k_0-1+l}}(\vx)$ is convex, which implies $\langle \vx_{k_0-1+l} - \vx^*, G_{k_0-1+l} - \nabla f_{\xi_{k_0-1+l}}(\vx^*) \rangle \geq 0$, $\mathds{P}(\nabla f_{\xi_{k_0-1+l}}(\vx^*) = \mathbf{0}) =1 $ by Assumption \ref{A4}.

Take expectation regarding to $\xi_{k_0-1+l}$, and use the fact that when $j > k_0$, $b_j > L$, when $l \geq 1$ and $0< \frac{\mu\eta}{b_{k_0-1+l}} <  \frac{\mu}{L} < 1$, then we can get
\begin{equation}
\begin{split}
     \mathds{E}_{\xi_{k_0-1+l}}\|\vx_{k_0+l}-\vx^*\|^2 
     & \leq  \|\vx_{k_0-1+l}-\vx^*\|^2 - \frac{\eta}{b_{\max}} \langle \vx_{k_0-1+l}-\vx^*, \nabla F(\vx_{k_0-1+l})\rangle \\
     & \leq (1-\frac{\mu\eta}{ b_{\max} }) \|\vx_{k_0-1+l}-\vx^*\|^2 \\
     & \leq \prod_{j=0}^{l} (1-\frac{\mu\eta}{ b_{\max} }) \|\vx_{k_0-1}-\vx^*\|^2\\
     & \leq \prod_{j=0}^{l} (1-\frac{\mu\eta}{ b_{\max} }) (\|\vx_0-\vx^*\|^2 + \eta^2 (\log(\frac{C^2}{b_0^2}) + 1)\\
     & \leq (\|\vx_0-\vx^*\|^2 + \eta^2 (\log(\frac{\eta^2L^2}{b_0^2}) + 1)) \exp( -\frac{\mu l}{b_{\max}} )
\end{split}
\end{equation}
where the second inequality is from the strong convexity of $F(\vx)$, i.e. $ \langle \vx-\vy, \nabla F(\vx) - \nabla F(\vy) \rangle \geq \mu \| \vx-\vy \|^2$ and $\nabla F(\vx^*)=\mathbf{0}$.
 From Lemma \ref{bmax}, we can give an upper bound for $b_{\max} = \max_{l\geq 0} b_{k_0+l} =  C + \frac{L}{\eta} (\|\vx_0-\vx^*\|^2 + \eta^2 (\log\frac{C^2}{b_0^2} + 1)) = \eta L + \frac{L}{\eta} (\|\vx_0-\vx^*\|^2 + \eta^2 (\log\frac{\eta^2L^2}{b_0^2} + 1)).$ \\

Then, take the iterated expectation and use Markov in inequality, with high probability $1-\delta_h$, 
\begin{equation*}
\begin{split}
     \|\vx_{k_0+l}-\vx^*\|^2 \leq \frac{1}{\delta_h} (\|\vx_0-\vx^*\|^2 + \eta^2 (\log \frac{\eta^2L^2}{b_0^2} + 1)) \exp( -\frac{\mu l}{b_{\max}} )
\end{split}
\end{equation*}
Then, after $M \geq \frac{ \eta L + \frac{L}{\eta} (\|\vx_0-\vx^*\|^2 + \eta^2 (\log \frac{\eta^2L^2}{b_0^2} + 1))}{\mu} \log \frac{\|\vx_0-\vx^*\|^2 + \eta^2 (\log \frac{\eta^2L^2}{b_0^2} + 1)}{\epsilon\delta_h}$ iterations, with high probability more than $1-\delta_h - \exp(-\frac{\delta^2}{2(N\gamma(1-\gamma)+\delta)}) $
\begin{equation*}
  \|\vx_{k_0+M} - \vx^*\|^2 \leq \epsilon
\end{equation*}
Otherwise, if $k_0 = 0$, i.e. $b_0 > \eta L$, then we use the same inequality as above,
\begin{equation*}
\begin{split}
    \mathds{E}_{\xi_{M-1}}\|\vx_M - \vx^* \|^2 & \leq (1-\frac{\mu}{b'_{\max}})\|\vx_{M-1} - \vx^*\|^2 \leq \| \vx_0 - \vx^*\|^2 \exp( -\frac{\mu M }{b'_{\max}})
\end{split}
\end{equation*}
Then, after $M \geq \frac{b'_{\max}}{\mu}\log \frac{\|\vx_0-\vx^*\|^2}{\epsilon\delta_h}$ iterations, by Markov's inequality, 
$$\mathds{P} (\|\vx_M -\vx^*\|^2 \geq \epsilon) \leq \frac{\mathds{E}\|\vx_M-\vx^*\|^2}{\epsilon}\leq \delta_h$$
where $b'_{\max} = b_0 + \frac{L}{\eta}\|\vx_0-\vx^*\|^2 $ is derived as follows:
\begin{equation*}
    \begin{split}
        \|\vx_{j+1} - \vx^*\|^2 & \leq \|\vx_{j} - \vx^*\|^2 - \frac{\eta}{L}\frac{\|G_j\|^2}{b_{j+1}} \\
        & \leq \|\vx_0-\vx^*\|^2 -   \frac{\eta}{L}\sum_{i=0}^{j}\frac{\|G_i\|^2}{b_{i+1}}
    \end{split}
\end{equation*}

Then, for any $j+1$:
\begin{equation}
   b_{j+1} = b_0 + \sum_{i=0}^{j} \frac{\|G_i\|^2}{b_i+b_{i+1}} \leq b_0 + \frac{L}{\eta} (\|\vx_0-\vx^*\|^2 - \|\vx_{j+1} - \vx^*\|^2)
\end{equation}
Plugging in the value, we can get $M \geq \frac{ b_0 + \frac{L}{\eta}\| \vx_0-\vx^*\|^2}{\mu}\log \frac{\| \vx_0-\vx^*\|^2}{\epsilon\delta_h}$.

\section{Proof of Theorem \ref{op2} and \ref{non-convex} in the batch Setting} \label{B}
\subsection{Proof of Theorem \ref{op2}}
\begin{lemma}\label{lem: eco}\textbf{(Co-coercivity with Strong Convexity)} \citep{bubeck2015convex} 
If $F(\vx)$ is $\mu-$strongly convex and $L-$smooth, then
$$\langle \nabla F(\vx) - \nabla F(\vy) , \vx-\vy \rangle \geq \frac{\mu L}{\mu+L}\|\vx-\vy\|^2 +
\frac{1}{\mu+L} \|\nabla F(\vx) - \nabla F(\vy)\|^2$$
\end{lemma}
\begin{proof}
  Let $\phi(\vx) = F(\vx) - \frac{\mu}{2}\|\vx\|^2$, then $\phi(\vx)$ is convex and $(L-\mu)-$smooth. By Lemma \ref{co-coercivity}, $$\langle \nabla \phi(\vx) - \nabla \phi(\vy) , \vx-\vy \rangle \geq 
  \frac{1}{L-\mu}\| \nabla \phi(\vx) - \nabla \phi(\vy) \|^2$$
  Plugging in $\nabla \phi(\vx) = \nabla F(\vx) - \mu \vx$, we have 
  \begin{equation*}
  \begin{split}
      \langle \nabla F(\vx) - \nabla F(\vy), \vx-\vy \rangle - \mu\|\vx-\vy\|^2 & \geq 
     \frac{1}{L-\mu} (\|\nabla F(\vx) - \nabla F(\vy)\|^2 + \mu^2 \|\vx-\vy\|^2 \\
      & - 2\mu \langle \nabla F(\vx) - \nabla F(\vy), \vx-\vy \rangle)
  \end{split}
  \end{equation*}
With simple algebra, we can get the result.
\end{proof}

By Lemma \ref{twocases}, after $N = \left \lceil \frac{\log(\eta^2(\mu+L)^2/4b_0^2)}{\log(1+4\mu^2\epsilon/(\mu+L)^2)} \right \rceil + 1$ iterations, if $\min_{0\leq i \leq N-1}\|\vx_i-\vx^*\|^2 > \epsilon$, then $\exists k_0 \leq N$, such that $k_0$ is the first index s.t. $b_{k_0} > \eta \frac{\mu+L}{2}$.

If $k_0 > 1$, since $F(\vx)$ is $\mu$-strongly convex and $L-$smooth, by Lemma \ref{lem: eco}, $$\langle \nabla F(\vx) - \nabla F(\vy), \vx-\vy \rangle \geq \frac{\mu L}{\mu + L}\|\vx-\vy\|^2 + \frac{1}{\mu+L}\|\nabla F(\vx) - \nabla F(\vy)\|^2.$$ 
For $j \geq 0$, we have $0 < \eta \frac{2\mu L}{\mu+L} \frac{1}{b_{k_0+j}} < \frac{4\mu L}{(\mu+L)^2} <1$ , since $2\mu L < \mu^2+ L^2$.

We divide the analysis into two situations to get a better bound instead of using $b_{\max}$ for all the following steps, which is different from the proof of Theorem \ref{sgdscalg}. First, assume that $\eta \frac{\mu+L}{2} < b_{k_0} < \eta L$ and after another $l$ iterations, $b_{k_0+l}$ is still less than $\eta L$, then $\|\vx_{k_0+l}-\vx^*\|^2$ is bounded as follows:
\begin{equation*}
\begin{split}
    \|\vx_{k_0+l}-\vx^*\|^2  & = \|\vx_{k_0-1+l}-\vx^*\|^2 + \frac{\eta^2}{b_{k_0+l}^2}\|\nabla F(\vx_{k_0-1+l})\|^2 \\  & - \frac{2\eta}{b_{k_0+l}}\langle \vx_{k_0-1+l}-\vx^*, \nabla F(\vx_{k_0-1+l})\rangle \\
    & \leq  (1-\frac{2\mu\eta L}{(\mu+L)b_{k_0+l}})\| \vx_{k_0-1+l}-\vx^*\|^2 \\
    & + \frac{\eta}{b_{k_0+l}}(\frac{\eta}{b_{k_0+l}}- \frac{2}{\mu+L})\|\nabla F(\vx_{k_0-1+l})\|^2\\
    & \leq \prod_{j=0}^{l}(1-\frac{2\mu\eta L}{(\mu+L)b_{k_0+j}})\|\vx_{k_0-1}-\vx^*\|^2 \\
    & \leq \exp(-\sum_{j=0}^l \frac{2\mu\eta L}{(\mu+L)\eta L})\|\vx_{k_0-1}-\vx^*\|^2\\
    & \leq \exp(-\frac{2\mu(l+1)}{\mu+L})\|\vx_{k_0-1}-\vx^*\|^2\\
\end{split}
\end{equation*}

where $\|\vx_{k_0-1}-\vx^*\|^2$ can be upper bounded according to Lemma \ref{boundset} with $C=\frac{\mu+L}{2}$:
$$\|\vx_{k_0-1} - \vx^*\|^2 \leq \|\vx_0-\vx^*\|^2 + \eta^2 (\log(\frac{(\mu+L)^2}{4b_0^2}) + 1) $$

Second, if $b_{k_0+M_0} > \eta L$, $M_0$ can be $0,1,2,\dots$, then for $l \geq 0$,
\begin{equation*}
    \begin{split}
    \|& \vx_{k_0+M_0+l}-\vx^*\|^2  = \|\vx_{k_0+M_0-1+l}-\vx^*\|^2 + \frac{\eta^2}{b_{k_0+l}^2}\|\nabla F(\vx_{k_0+M_0-1+l})\|^2 \\
     & - \frac{2\eta}{b_{k_0+M_0+l}}\langle \vx_{k_0+M_0-1+l}-\vx^*, \nabla F(\vx_{k_0+M_0-1+l})\rangle \\
     & \leq \| \vx_{k_0+M_0-1+l}-\vx^*\|^2 + (\frac{\eta^2L}{b_{k_0+M_0+l}^2}- \frac{2\eta}{b_{k_0+M_0+l}}) \langle \vx_{k_0+M_0-1+l}- \vx^*, \nabla F(\vx_{k_0+M_0-1+l})\rangle\\
     & \leq (1- \frac{\mu\eta}{b_{\max}})\| \vx_{k_0+M_0+l} - \vx^*\|^2\\
     & \leq \exp(-\frac{\mu\eta l}{b_{\max}})\|\vx_{k_0+M_0} - \vx^*\|^2\\
     & \leq \exp(-\frac{\mu\eta l}{b_{\max}}) \exp(-\frac{2\mu(M_0+1)}{\mu+L})\|\vx_{k_0-1}-\vx^*\|^2
     \end{split}
\end{equation*}

where $b_{\max}$ can be upper bounded according to Lemma \ref{bmax}, here $b_{\max} \leq \eta L + \frac{L}{\eta}\|\vx_{k_0+M_0} - \vx^*\|^2$.

Once $b_t > \frac{\mu + L}{2} > \frac{L}{2}$, by Lemma \ref{descent}, AdaGrad-Norm is indeed a decent algorithm for $\|\vx_j-\vx^*\|^2$, so $\|\vx_{k_0+M_0-1}-\vx^*\|^2 \leq \|\vx_{k_0-1}-\vx^*\|^2$.
Hence, $$b_{\max} \leq \eta L + \frac{L}{\eta}\|\vx_{k_0-1}-\vx^*\|^2$$

Combining the two situations above, we have
\begin{equation*}
    \begin{split}
    \|\vx_{k_0+M} - \vx^*\|^2 & \leq \exp(- M \min \left\{ \frac{\mu\eta}{b_{\max}}, \frac{2\mu}{\mu+L} \right\}) \|\vx_{k_0-1}-\vx^*\|^2\\
    & \leq \exp(- M \min \left\{ \frac{\mu}{L(1+\Delta/\eta^2)}, \frac{2\mu}{\mu+L} \right\}) \|\vx_{k_0-1}-\vx^*\|^2
    \end{split}
\end{equation*}
where $\Delta = \|\vx_{k_0-1}-\vx^*\|^2$.

After $M\geq \max \left \{ \frac{L(1+\Delta/\eta^2)}{\mu}, \frac{\mu+L}{2\mu} \right\} \log \frac{\Delta}{\epsilon}-1$ iterations,
$$\|\vx_{k_0+M} - \vx^*\|^2 \leq \epsilon $$

Otherwise, if $k_0=1$, then
$$ \|\vx_M - \vx^*\|^2 \leq \exp(-\sum_{j=1}^M \min \left\{ \frac{\mu\eta}{b'_{\max}}, \frac{2\mu}{\mu+L} \right\}) \|\vx_0 - \vx^*\|^2 $$
where $b'_{\max} = \eta L + \frac{L}{\eta}\|\vx_0 -\vx^*\|^2$.\\
Then, after $M \geq \max\left\{ \frac{L(1+\|\vx_0-\vx^*\|^2/\eta^2)}{\mu}, \frac{\mu+L}{2\mu} \right\} \log \frac{\|\vx_0-\vx^*\|^2}{\epsilon}$ iterations, we can assure that
$$\|\vx_M - \vx^*\|^2 \leq \epsilon$$

\subsection{Proof of Theorem \ref{non-convex}}
By Lemma \ref{twocases}, after $N \geq \frac{\log(\eta^2L^2/b_0^2)}{\log(1+2\mu\epsilon/(\eta L)^2)}$ iterations, if $\min_{0\leq i \leq N-1}F(\vx_i)-F^* > \epsilon$, then $\exists k_0 \leq N$, such that $k_0$ is the first index s.t. $b_{k_0} > \eta L$.\\
If $k_0 > 1$, then for $j \geq 0$, from Assumption \ref{A2}, we have
\begin{equation}
    \begin{split}
        F(\vx_{k_0+j}) & \leq F(\vx_{k_0+j-1}) - \frac{\eta}{b_{k_0+j}}(1-\frac{\eta L}{2b_{k_0+j}})\|\nabla F(\vx_{k_0+j-1})\|^2 \\
        & \leq  F(\vx_{k_0+j-1}) - \frac{\eta}{2b_{k_0+j}} \|\nabla F(\vx_{k_0+j-1})\|^2\\
        & \leq  F(\vx_{k_0+j-1}) + \frac{\mu \eta}{b_{k_0+j}}(F^* - F(\vx_{k_0+j-1}))
    \end{split}
\end{equation}
The last inequality is from $\mu-$PL inequality (Assumption \ref{A1b}): $-\|F(\vx)\|^2 \leq 2\mu (F*-F(\vx_j)), \forall x$. Then, add $-F^*$ on both sides, we can get 
\begin{equation}
    F(\vx_{k_0+j})-F^* \leq (1-\frac{\mu\eta}{b_{k_0+j}})(F(\vx_{k_0+j-1})-F^*)
\end{equation}
 Since $b_{k_0+j} >\eta L \geq \eta \mu$, $1-\frac{\mu\eta}{b_{k_0+j}} \in (0,1)$ holds for all $j\geq0$, it is a contraction at every step. Then,
\begin{equation}
    \begin{split}
        F(\vx_{k_0+j}) - F^* & \leq (\prod_{l=0}^{j}(1-\frac{\mu\eta}{b_{k_0+l}}))(F(\vx_{k_0-1})-F^*)\\
        & \leq \exp(-\sum_{l=0}^{j}\frac{\mu\eta}{b_{k_0+l}})(F(\vx_{k_0-1}) - F^*)\\
        & \leq \exp(-\sum_{l=0}^{j}\frac{\mu\eta}{b_{k_0+l}})(F(\vx_0) - F^* + \frac{\eta^2L}{2}(1+\log(\frac{b_{k_0-1}^2}{b_0^2}))\\
    \end{split}
\end{equation}
where we use the fact that $1-x\leq e^{-x}, \forall x \in (0,1)$ and the lemma in \cite{ward2018adagrad}: $F(\vx_{k_0-1}) \leq F(\vx_0) + \frac{\eta^2L}{2}(1+\log(\frac{b_{k_0-1}^2}{b_0^2}))$.\\
The upper bound of $b_j$ is also from \cite{ward2018adagrad}:
\begin{equation}
    b_{\max} = b_{k_0-1} + \frac{2}{\eta}(F_{k_0-1} - F^*) \leq \eta L + \frac{2}{\eta}( F(\vx_0) -F^* + \frac{\eta^2L}{2}(1+\log(\frac{\eta^2L^2}{b_0^2})))
\end{equation}
Then,  $$F(\vx_{k_0+M-1}) - F^* \leq \exp(-\frac{\mu\eta M}{b_{\max}})(F(\vx_0) - F^* + \frac{\eta^2L}{2}(1+2\log\frac{\eta L}{b_0})) $$
Hence, we need $$M \geq \frac{b_{\max}}{\mu\eta }\log\frac{F(\vx_0) - F^* + \frac{\eta^2L}{2}(1+2\log\frac{\eta L}{b_0})}{\epsilon}$$
It is sufficient that
$$M \geq \frac{\eta L + \frac{2}{\eta}( F(\vx_0) -F^* + \frac{\eta^2L}{2}(1+\log \frac{\eta^2L^2}{b_0^2}))}{\mu\eta }\log\frac{F(\vx_0) - F^* + \frac{\eta^2L}{2}(1+2\log\frac{\eta L}{b_0})}{\epsilon}$$
Then, $$\min_{0\leq i \leq N+M-1}F(\vx_i)-F^* \leq \epsilon$$
where $N = \left \lceil \frac{\log(\eta^2L^2/b_0^2)}{\log(1+2\mu\epsilon/ (\eta L)^2)} \right \rceil +1 $.\\\\
Otherwise, if $k_0 = 1$, the upper bound of $b_j$ degenerates to
$$b'_{\max} = b_0 + \frac{2}{\eta}(F(\vx_0)-F^*) $$
Then, using the same procedure, we have 
\begin{equation}
\begin{split}
     F(\vx_M) - F^* & \leq \exp(-\sum_{k=0}^{M-1}\frac{\mu\eta}{b_{k+1}})(F(\vx_0)-F^*) \\
     & \leq \exp(\frac{-\mu\eta M }{b_{\max}}) (F(\vx_0)-F^*)\\
\end{split}
\end{equation}
Once the number of iterations satisfies $$M \geq  \frac{b'_{\max}}{\mu\eta} \log\frac{F(\vx_0)-F^*}{\epsilon}
= \frac{b_0+\frac{2}{\eta}(F(\vx_0)-F^*)}{\mu\eta}\log\frac{F(\vx_0)-F^*}{\epsilon}$$
we can get the expected result: $F(\vx_M) - F^* \leq \epsilon $.

\section{Proof of Lemmas in Stage I}\label{C}

\subsection{Proof of Lemma \ref{highp}}
\begin{lemma}\label{bernstein}
\textbf{(Bernstein's Inequality)} \citep{wainwright2019high} Let $X$ be a random variable, $\mathds{E}[X] = \mu$, $Var(X)=\sigma^2$, if $X$ satisfies Bernstein condition with parameter $b>0$, i.e. if $|\mathds{E}(X-\mu)^k| \leq \frac{1}{2}k!\sigma^2b^{k-2}, \forall k\geq 2$, then 
$$\mathds{P}(|X-\mu|\geq t) \leq 2 \exp(-\frac{t^2}{2(\sigma^2+bt)}) $$
\end{lemma}

\begin{lemma}\label{bernoulli} \citep{wainwright2019high}
Let $X_i \sim  Bernoulli(p), i.i.d. \forall i=1,2,\dots n$, and $X = \sum_{i=1}^n X_i$. Since $X_i \in [0,1]$, $\{X_i\}$ satisfy Bernstein condition, then
$$P(|X-np|>t) \leq 2\exp(-\frac{t^2}{2(np(1-p)+t)})$$
\end{lemma}
\paragraph{Proof of Lemma \ref{highp}}
If $\min_j \|\vx_j - \vx^*\|^2 \leq \epsilon$, we are done.\\
Otherwise, we have $\|\vx_j-\vx^*\|^2 > \epsilon, \forall j = 0,1,2,\dots,N$. Assume that $F(\vx)$ satisfies $(\epsilon, \alpha, \gamma)-$ RUIG (Assumption \ref{A3}), we can use independent identical Bernoulli random variables $\{Z_j\}$ to represent them with the following distribution:
\begin{equation}
    Z_j = \left\{
    \begin{array}{cc}
    1 & if ~ \|\nabla f_{\xi_j}(\vx_j)\|^2 \geq \alpha \|\vx_j - \vx^*\|^2 \\
    0 & else 
    \end{array}
    \right.
\end{equation}
where $\mathds{P}(Z_j = 1) = \gamma, \forall j$. Note that the RUIG assumption is for any fixed $x$ (conditional on $\vx$), the probability distribution is over the random variable $i$ (or $\xi_i$) (but not over $\vx$). Every index $\xi_i$ is sampled independently and uniformly at each iteration, so random variables $\{Z_i\}$ are independent. Then, from Lemma \ref{bernoulli} and let $Z=\sum_j Z_j$, with high probability bigger than $1- \exp(-\frac{\delta^2}{2(N\gamma(1-\gamma)+\delta)})$, $Z \geq \gamma N -\delta, \forall N$. Thus, after $N \geq \frac{C^2-b_0^2}{\alpha\gamma\epsilon}+\frac{\delta}{\gamma}$ iterations, with $1- \exp(-\frac{\delta^2}{2(N\gamma(1-\gamma)+\delta)})$, we have
$$b_N^2 = b_0^2 + \sum_{i=0}^{N-1} \| \nabla f_{\xi_i}(\vx_i)\|^2 > b_0^2 + (\gamma N -\delta)\alpha \epsilon \geq C^2$$

Note that even if in the case that there is some correlation between Bernoulli random variables, since each of them is sub-Gaussian with $\sigma = 0.5$, then the upper bound of the sub-Gaussian parameter of the sum of them is $0.5N$, so the worst-case variance is $0.25N^2$. Hence, the result still holds under this setting.


\subsection{Proof of Lemma \ref{twocases}}
\begin{enumerate}[label=(\alph*)]
    \item If $b_0 > C$, we are done.\\
Otherwise if $b_0 < C$, and after $N \geq \frac{\log(C^2/b_0^2)}{\log(1+\mu^2\epsilon/ C^2)}$ iterations, $b_N < C$ and $\min_{0\leq i \leq N-1} \|\vx_i - \vx^*\|^2 > \epsilon$. Since $F(\vx)$ is $\mu-$ strongly convex, $\epsilon < \|\vx_i - \vx^*\|^2 \leq \frac{1}{\mu^2} \|\nabla F(\vx_i) - \nabla F(\vx^*)\|^2, \forall \vx_i$ and $\nabla F(\vx^*) =\mathbf{0}$. Then,
\begin{equation}
\begin{split}
       b_{N}^2 & = b_{N-1}^2 +\|\nabla F(\vx_{N-1})\|^2 \\
        & = b_{N-1}^2(1 + \frac{\|\nabla F(\vx_{N-1})\|^2}{b_{N-1}^2})\\
        & \geq b_0^2 \prod_{j=0}^{N-1} (1+ \frac{\|\nabla F(\vx_j)\|^2}{b_j^2})\\
        & \geq b_0^2(1+\frac{\mu^2\epsilon}{C^2})^N \geq C^2
\end{split}
\end{equation}
Contradiction! Hence, at least one of  $\min_{0\leq i \leq N-1} \|\vx_i - \vx^*\|^2 \leq \epsilon$ or $b_N > C$ holds. When $\mu$ is small and $C$ is big, we have $\log(1+\frac{\mu^2\epsilon}{C^2}) \approx \frac{\mu^2\epsilon}{C^2}$.

\item With PL inequality $\frac{1}{2\mu}\|\nabla F(\vx)\|^2 \geq F(\vx) - F(\vx^*)$ instead of $\mu$-strongly convex assumption, if $\min_{0\leq i \leq N-1} F(\vx_i) - F^* > \epsilon$ and $b_N<C$, then after $N \geq \frac{\log(C^2/b_0^2)}{\log(1+2\mu\epsilon/ C^2)}$ iterations,
$b_N^2 \geq b_0^2 (1+\frac{2\mu\epsilon}{C^2})^N \geq C^2$, contradiction! Hence, either $\min_{0\leq i \leq N-1} F(\vx_i) - F^* \leq \epsilon$ or $b_N>C$.
\end{enumerate}

\subsection{Proof of Lemma \ref{boundset}}

\begin{lemma}\label{co-coercivity}
\textbf{(Co-coercivity)} \citep{Needell2016}
For a $L-$smooth  convex function $F(\vx)$, $\forall \vx,\vy$ 
$$\|\nabla F(\vx) - \nabla F(\vy) \|^2 \leq L \langle \vx-\vy, \nabla F(\vx) - \nabla F(\vy) \rangle$$
\end{lemma}

\begin{lemma}\label{log}
\textbf{(Integral lemma)} \citep{ward2018adagrad}
For any non-negative sequence $a_1,\dots,a_T$, such that $a_1 \geq 1$, 
\begin{align}
      \sum_{l=1}^T \frac{a_l}{\sum_{i=1}^{l}a_i} &\leq \log(\sum_{i=1}^T a_i) + 1 \\
    \sum_{l=1}^T \frac{a_l}{\sqrt{\sum_{i=1}^{l}a_i}} &\leq 2\sqrt{\sum_{i=1}^Ta_i}  
\end{align}
\end{lemma}
\begin{proof}
The lemma can be proved by induction. Besides, we can take above sums as Riemman sums, then the sums should be proportional to integrals, $\log(x)$ and 2$\sqrt{x}$, respectively.
\end{proof}

\paragraph{Proof of Lemma \ref{boundset}} With above two lemmas, we can bound $\|\vx_{J-1} - \vx^*\|^2$ as follows:
\begin{equation*}
    \begin{split}
        \|\vx_{J-1} - \vx^*\|^2 & = \|\vx_{J-2} - \frac{\eta G_{J-2}}{b_{J-1}} - \vx^*\|^2\\
        & =  \|\vx_{J-2}-\vx^*\|^2 + \|\frac{\eta G_{J-2}}{b_{J-1}} \|^2 - 2 \eta \langle \frac{ G_{J-2}}{b_{J-1}} ,\vx_{J-2}-\vx^*\rangle  \\
        & \leq  \|\vx_{J-2}-\vx^*\|^2 + \|\frac{\eta G_{J-2}}{b_{J-1}} \|^2 - \frac{2\eta}{b_{J-1}L}\|G_{J-2} - \nabla f_{\xi_{J-2}}(\vx^*) \|^2 \\
        & \leq  \|\vx_{J-2}-\vx^*\|^2 + \frac{\eta^2\|G_{J-2}\|^2}{b_{J-1}^2} \\
        & \leq  \|\vx_0-\vx^*\|^2 + \eta^2 \sum_{j=0}^{J-2} \frac{\|G_j\|^2}{b_{j+1}^2} \\
        & \leq  \|\vx_0-\vx^*\|^2 + \eta^2 \sum_{j=0}^{J-2} \frac{\|G_j\|^2 / b_0^2}{\sum_{l=0}^{j}\|G_l\|^2 / b_0^2}\\
    \end{split}
\end{equation*}
\begin{equation*}
    \begin{split}
        & \leq  \|\vx_0-\vx^*\|^2 + \eta^2 (\log(\sum_{j=0}^{J-2}  \|G_j\|^2/b_0^2) + 1)\\
        & \leq  \|\vx_0-\vx^*\|^2 + \eta^2 (\log \frac{C^2}{b_0^2} + 1)
    \end{split}
\end{equation*}
where the first inequality is from the co-coercivity (Lemma \ref{co-coercivity}) and Assumption \ref{A4} $\mathds{P}(\nabla f_{\xi_{J-2}}(\vx^*)=\mathbf{0})=1$; last second inequality is from lemma \ref{log} and the last inequality is from the assumption that $J$ is the first index s.t. $ b_J >  C$.

\section{Proof of Lemmas in Stage II}\label{D}
\subsection{Proof of Lemma \ref{bmax}}
Since $b_J>\eta L$, we have the following bound for $\|\vx_{J+l} - \vx^*\|^2$:
\begin{equation}
\begin{split}
    \|\vx_{J+l} -\vx^*\|^2 &  = \|\vx_{J+l-1}-\vx^*\|^2 + \frac{\eta^2 \|G_{J+l-1}\|^2}{b_{J+l}^2} \\ 
    & - \frac{2\eta}{b_{J+l}}\langle G_{J+l-1}-\nabla f_{\xi_{J+l-1}}(\vx^*), \vx_{J+l-1} - \vx^*\rangle   \\
    & \leq  \|\vx_{J+l-1}-\vx^*\|^2 + \|\frac{\eta G_{J+l-1}}{b_{J+l}} \|^2 - \frac{2\eta}{b_{J+l}L}\|G_{J+l-1}-\nabla f_{J+l-1}(\vx^*)\|^2\\
    & \leq \|\vx_{J+l-1}-\vx^*\|^2 + \frac{\eta}{b_{J+l}}(\frac{\eta}{b_{J+l}} -\frac{2}{L}) \|G_{J+l-1}\|^2\\
    & \leq \|\vx_{J+l-1}-\vx^*\|^2 - \frac{\eta}{L}\frac{\|G_{J+l-1}\|^2}{b_{J+l}}\\
    & \leq \|\vx_{J-1}-\vx^*\|^2 - \frac{\eta}{L} \sum_{j=0}^{l}\frac{\|G_{J+j-1}\|^2}{b_{J+j}}
\end{split}
\end{equation}inequalities are from $f_{\xi_{J+l-1}}(\vx)$ is $L$-smooth (Assumption \ref{A2}) and co-coercivity (Lemma \ref{co-coercivity}). Then, we have the bound of the sum:
\begin{equation}
    \sum_{j=0}^{l}\frac{\|G_{J+j-1}\|^2}{b_{J+j}} \leq  \frac{L}{\eta}(\|\vx_{J-1} -\vx^*\|^2 -  \|\vx_{J+l}-\vx^*\|^2)
\end{equation}
Therefore, $b_{\max}$ is bounded as follows:
\begin{equation}
    \begin{split}
        b_{J+l} & = b_{J+l-1} + \frac{\|G_{J+l-1}\|^2}{b_{J+l}+b_{J+l-1}}\\
        & \leq b_{J-1} +\sum_{j=1}^{l} \frac{\|G_{J+j-1}\|^2}{b_{J+j}} \\
        & \leq C + \frac{L}{\eta}\|\vx_{J-1} -\vx^*\|^2 \\
        & = C + \frac{L}{\eta} (\|\vx_0-\vx^*\|^2 + \eta^2 (\log\frac{C^2}{b_0^2} + 1))
    \end{split}
\end{equation}
where $\|\vx_{J-1} -\vx^*\|^2 \leq \|\vx_0-\vx^*\|^2 + \eta^2 (\log \frac{C^2}{b_0^2} + 1)$ is from Lemma \ref{boundset}.

\subsection{Proof of Lemma \ref{descent}}
Use similar technique as above,
\begin{align*}
        \|\vx_j - \vx^*\|^2 
        & \leq \|\vx_{j-1} - \vx^*\|^2 + (\frac{\eta^2}{b_{j}^2} - \frac{2\eta}{b_{j}L} ) \|G_{j-1}\|^2\\
        & = \|\vx_{j-1} - \vx^*\|^2 - \frac{\eta}{b_{j}}(\frac{2}{L} - \frac{\eta}{b_{j}})\|G_{j-1}\|^2 \leq  \|\vx_{j-1} - \vx^*\|^2
\end{align*}
where the first inequality if from $f_j(\vx)$ is $L$-smooth (Assumption \ref{A2}), $\nabla f_{j-1}(\vx^*) =\mathbf{0}$ (Assumption \ref{A4}) and  Lemma \ref{co-coercivity}. Therefore, once $b_j > \eta L/2$, AdaGrad-Norm is a descent algorithm.

\section{More Numerical Experiments}\label{E}
\subsection{Numerical Experiments of AdaGrad-Norm with Extreme Initialization} \label{E1}

In this section, we demonstrates the numerical experiments of AdaGrad-Norm with $\vx_0=\mathbf{0}$ (stochastic: Figure \ref{fig:in0}; batch: Figure \ref{fig:in0batch}) and the extreme case (stochastic: Figure \ref{fig:inbad}; batch: Figure \ref{fig:inbadbatch}): $\vx_0$ is far away from $\vx^*$ and $\|\vx_0\|$ is large. Then, we tune the hyperparameter $\eta$ in the extreme case with $\eta = \Theta(\|\vx_0-\vx^*\|^2)$ (stochastic: Figure \ref{fig:inbadtune}) and  $\eta = \Theta(\|\vx_0-\vx^*\|)$ (batch: Figure \ref{fig:inbadtunebatch}). In these figures, the x-axis represents iteration $t$ while y-axis is the approximation error $\|\vx_t-\vx^*\|^2$ in log scale for the first and third columns and it is $b_t$ for the second and fourth columns.

We show that when starting from $\vx_0 = \mathbf{0}$, the result is close to the experiment we show in Figure \ref{fig:ls}. When initialize $\vx_0$ with extremely bad one, $\vx_0 = 100 * \vw_0$, where $\vw_0$ is a randomly generated vector $\vw_0$ and $\vw_0 \sim \mathcal{N}(\mathbf{0},\mI)$, AdaGrad-Norm takes much more iterations than before. However, after tuning $\eta = 10000$ in stochastic setting and $\eta = 100$ in batch setting, the convergence rate of AdaGrad-Norm is better again. In this case, $b_0$ plays a small role.
\begin{figure}[H]
    \centering
    \includegraphics[width=\linewidth]{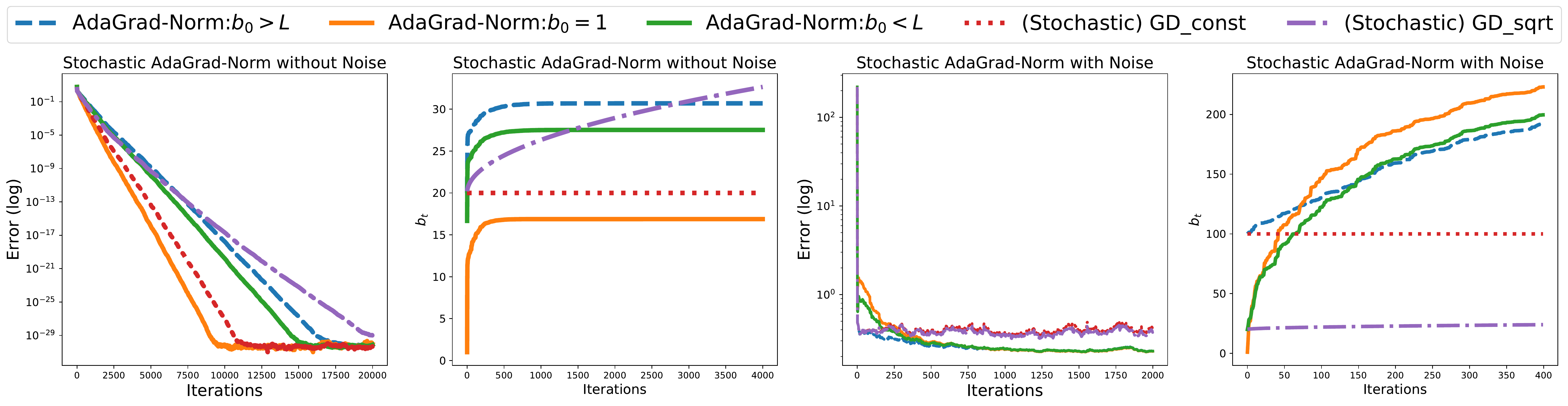}
    \caption{Error and growth of $b_t$ with $\vx_0 =\mathbf{0}$ and $\eta=1$ in stochastic setting.}
    \label{fig:in0}
\end{figure}
\begin{figure}[H]
    \centering
    \includegraphics[width=\linewidth]{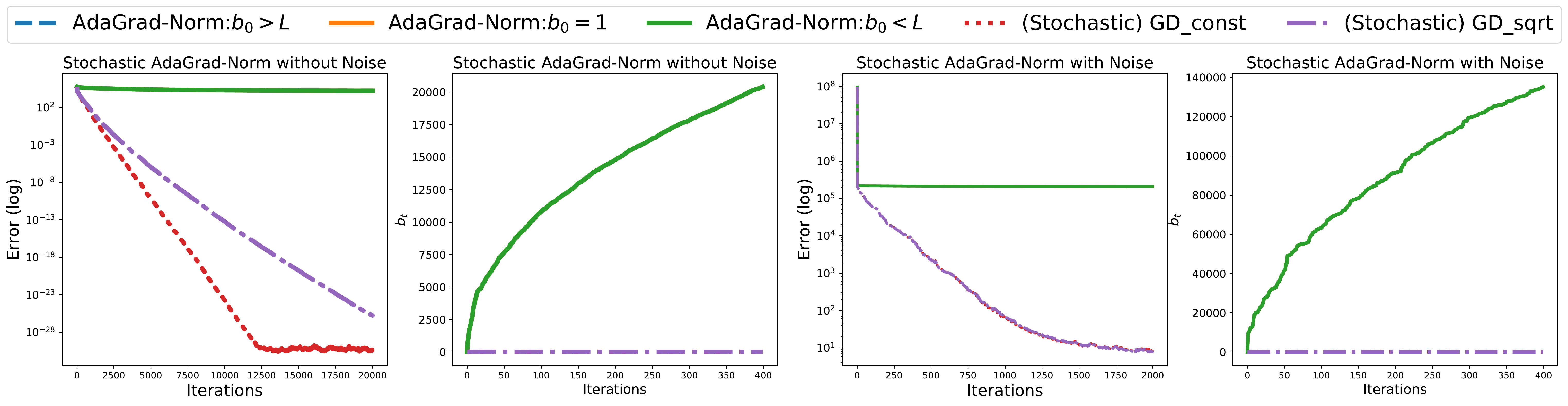}
    \caption{Error and growth of $b_t$ with extremely bad initialization and $\eta=1$ in stochastic setting.}
    \label{fig:inbad}
\end{figure}
\begin{figure}[H]
    \centering
    \includegraphics[width=\linewidth]{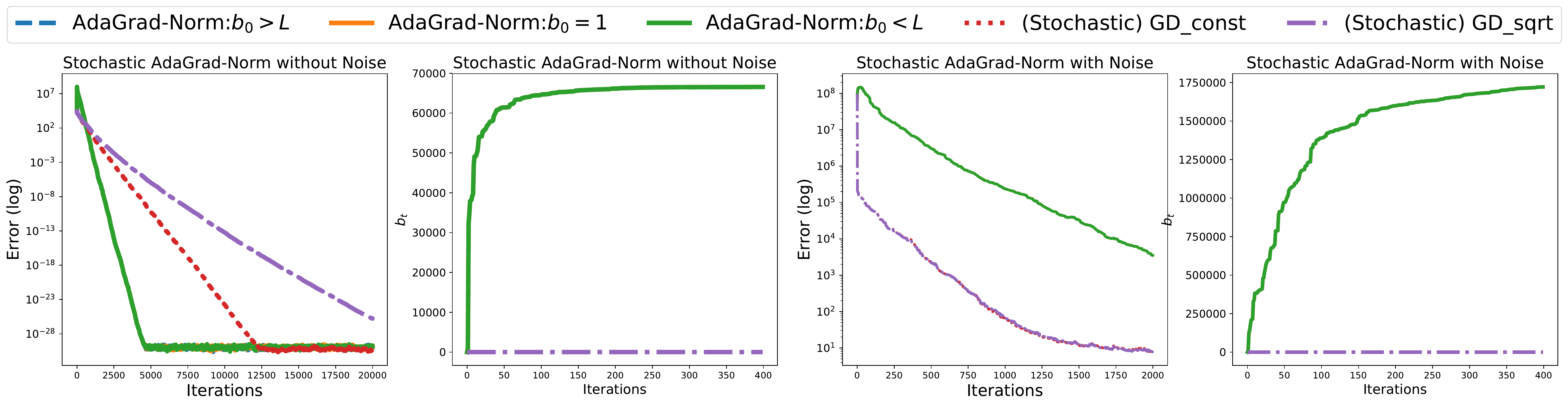}
    \caption{Error and growth of $b_t$ with extremely bad initialization and tuning $\eta = \Theta(\|\vx_0-\vx^*\|^2)$ in stochastic setting.}
    \label{fig:inbadtune}
\end{figure}

\begin{figure}[H]
    \centering
    \includegraphics[width=\linewidth]{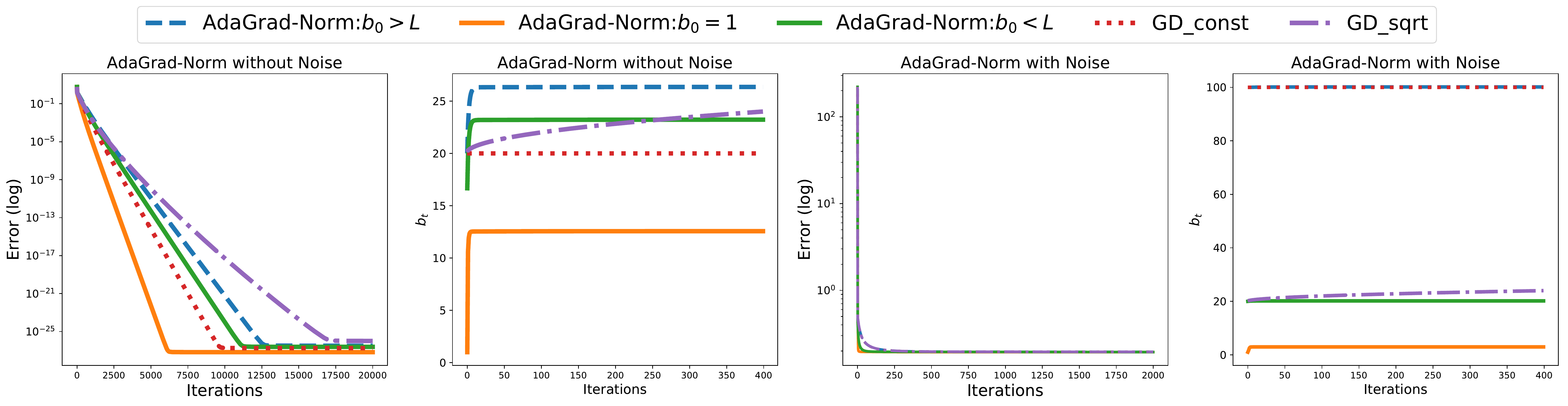}
    \caption{Error and growth of $b_t$ with $\vx_0 =\mathbf{0}$ and $\eta=1$ in batch setting.}
    \label{fig:in0batch}
\end{figure}
\begin{figure}[H]
    \centering
    \includegraphics[width=\linewidth]{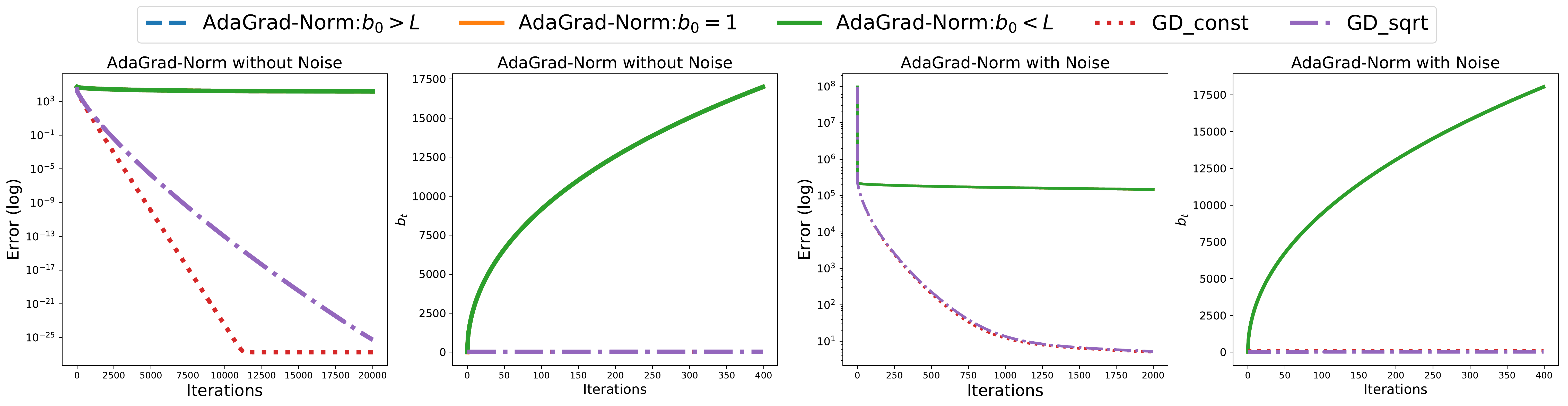}
    \caption{Error and growth of $b_t$ with extremely bad initialization and $\eta=1$ in batch setting.}
    \label{fig:inbadbatch}
\end{figure}
\begin{figure}[H]
    \centering
    \includegraphics[width=\linewidth]{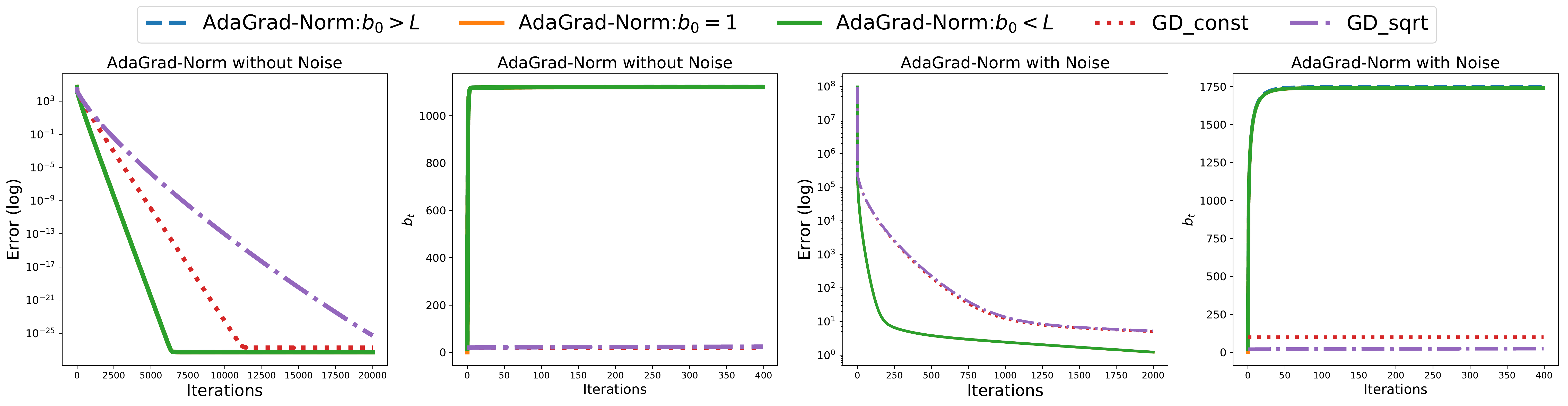}
    \caption{Error and growth of $b_t$ with extremely bad initialization and tuning $\eta = \Theta(\|\vx_0-\vx^*\|)$ in batch setting.}
    \label{fig:inbadtunebatch}
\end{figure}

\subsection{Numerical Experiment of Two Layer  Neural Networks} \label{E2}
We implement AdaGrad-Norm in a two-layer network. The experiment is mainly to show the stochastic AdaGrad-Norm (black curve) converges with a linear rate. We first define loss function as in \cite{du2018gradient}:
$$L(\mW) = \frac{1}{2n}\sum_{i=1}^n\left(\frac{1}{\sqrt{m}}\sum_{r=1}^{m}a_r
     \sigma(\langle{\vw_r(k), \vx_{i} \rangle})- \vy_{i} \right)^2  $$
     where $\sigma$ is a ReLU activation function; $n$ is size of data; $m$ is the width for the one-hidden layer.   For our implementation, we set $n=100$, $m=200$ and $d=10$.
Set mini-batch size $20$ for each iteration and the effective stepsize of AdaGrad-Norm with $100/b_t$ and $b_0=0.1$.  We also run vanilla SGD (blue curve)  with  $\eta_t = 100$. For details, see the code here \footnote{\url{https://colab.research.google.com/drive/1kv-XwUxvSogVfNyTO2w1aAoqlS2chlYH}}. Figure \ref{fig:over} (left) clearly illustrates that AdaGrad-Norm (black curve)  converges linearly.
 Figure \ref{fig:over} (right) shows the norm of the gradients at the first few iterations by AdaGrad-Norm are often big enough to accumulate to exceed $\eta L$, which empirically verifies Assumption \ref{A3}.

\begin{figure}[H]
    \centering
    \includegraphics[width=1.1\linewidth]{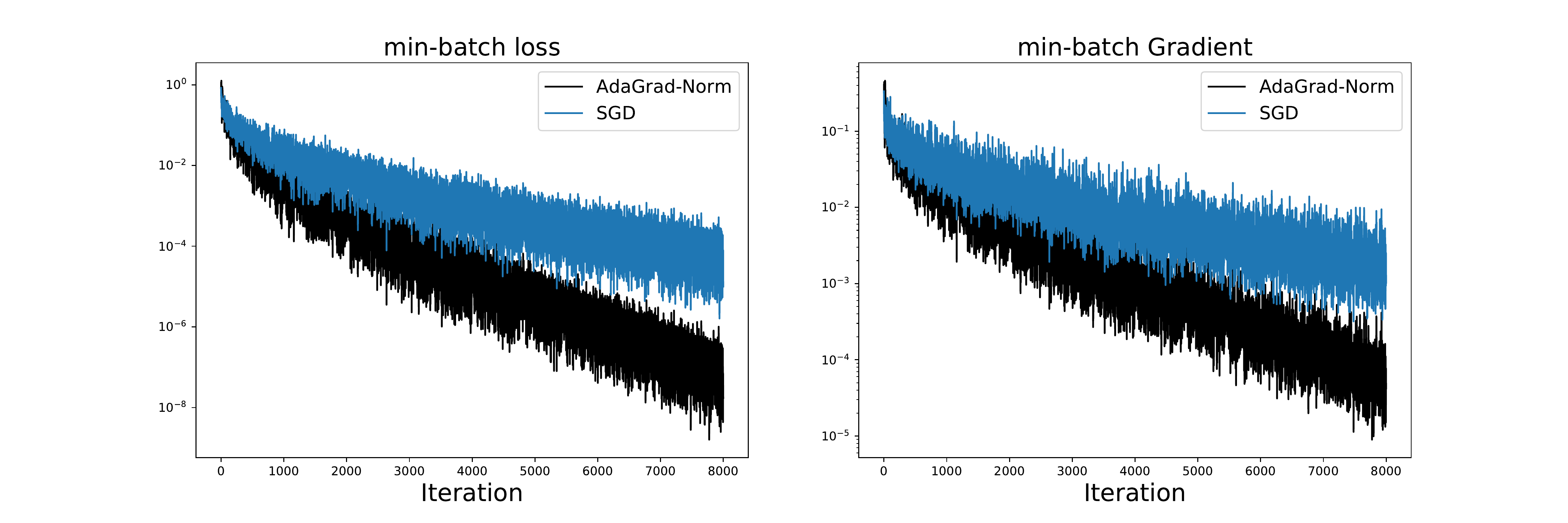}
    \caption{Error and the norm of gradient in a two-layer neural network}
    \label{fig:over}
\end{figure}

\end{document}